%% file: main.tex
\documentclass[reqno]{new_tlp}
\usepackage{mathptmx}
\usepackage{amsmath}
\usepackage{amssymb}
\usepackage{comment}
\usepackage{hyperref}
\usepackage[T1]{fontenc}

\let\proof\relax

\usepackage{amsthm}

\usepackage{enumitem}

\usepackage{listings}

\newcommand\bcmdtab{\noindent\bgroup\tabcolsep=0pt%
  \begin{tabular}{@{}p{10pc}@{}p{20pc}@{}}}
\newcommand\ecmdtab{\end{tabular}\egroup}

\newcommand\code[1]{\texttt{#1}}

\usepackage{color}

\input{command}

  \title[Paraconsistency and Word Puzzles]
        {Paraconsistency and Word Puzzles}

  \author[Tiantian et al.]
         {TIANTIAN GAO, PAUL FODOR and MICHAEL KIFER \\
         Stony Brook University, New York, USA\\
         \email{\{tiagao, pfodor, kifer\}@cs.stonybrook.edu}
         }
         
\jdate{April 2016}
\pubyear{2016}
\pagerange{\pageref{firstpage}--\pageref{lastpage}}
\doi{S1471068401001193}

\newtheorem{theorem}{Theorem}
\newtheorem{lemma}[theorem]{Lemma}

\theoremstyle{definition}
\newtheorem{definition}{Definition}

\newcounter{principlecounter}
\newenvironment{principle}[1]{
  \refstepcounter{principlecounter}
  \noindent
  \textbf{Principle \arabic{principlecounter}:\ }
  \textbf{#1}\par
  
}

\newcommand{\epiarrow}{\ensuremath{<\!\!\sim}}
\newcommand{\epimodel}{\ensuremath{|\!\!\!\approx}}
\newcommand{\epinotmodel}{\ensuremath{|\!\!\!\not\approx}}
\newcommand{\epicmodel}{\ensuremath{|\!\!\!\approx_{\cS}}}
\newcommand{\epinotcmodel}{\ensuremath{|\!\!\!\not\approx_{\cS}}}

\newcommand{\TV}[1]{\textnormal{\textbf{#1}}}

\begin{document}
%\nocite{*}% includes all entries of BibTeX database into the list of references.

\label{firstpage}

\maketitle

\begin{abstract}
Word puzzles and the problem of their representations in logic languages have
received considerable attention in the last decade \cite{PonnuruFMT04,Shapiro11,BaralD12,schwitter2013jobs}.  
Of special interest is the problem of generating such representations
directly from
natural language (NL) or controlled natural language (CNL).  An interesting
variation of this problem, and to the best of our knowledge, scarcely
explored variation in this context, is when the input information is
inconsistent. In such situations, the existing encodings of word puzzles
produce inconsistent representations and break down. In this paper, we
bring the well-known type of paraconsistent logics, called 
\emph{Annotated Predicate Calculus} (APC) \cite{kifer1992logic}, 
to bear on the problem.  We introduce a new kind of non-monotonic semantics for APC, called 
\emph{consistency preferred stable models} and argue that 
it makes APC into a suitable platform for
dealing with inconsistency in word puzzles and, more generally, in NL
sentences. We also devise a number of general principles to help the user
choose among the different representations of NL sentences, which might
seem equivalent but, in fact, behave differently when inconsistent
information is taken into account. These principles can be incorporated
into existing CNL translators, such as 
Attempto Controlled English (ACE) \cite{fuchs2008attempto} and 
PENG Light \cite{white2009update}.  
Finally, we show that APC with the consistency preferred stable model semantics can be equivalently embedded in ASP with preferences over stable models, and we use this embedding to implement this version of APC in Clingo \cite{gekakaosscsc11a} and its Asprin add-on \cite{brewka2015asprin}. 

To appear in Theory and Practice of Logic Programming (TPLP).
\end{abstract}


%\tableofcontents

\input{introduction}

\input{apc}

\input{apclp}

\input{apc_asp}

\input{puzzle}

\input{puzzle_principles}

\input{related}

\input{conclusion}

\bibliographystyle{acmtrans}
\bibliography{main}

\newpage

\appendix
\input{jobs_puzzle_apc}

\input{zebra_puzzle_apc}

\input{marathon_puzzle_apc}

\label{lastpage}
\end{document}

%% file: command.tex
\newtheorem{dummylemma}{BLA}%[chapter]
\newtheorem{example}{Example}%[chapter]{\itshape}{\rmfamily}

\newcommand{\ignore}[1] {}

\DeclareTextFontCommand{\verdana}{\fontencoding{T1}\fontfamily{verdana}\selectfont}

\newlength\Rulewid
\newtheorem{DefinitioN}[dummylemma]{\em Definition}%[section]
\newenvironment{definition-1}[1]
                {\begin{DefinitioN}[#1]\rm}{\hfill $\Box$\end{DefinitioN}}
\newcommand{\cA}{{\cal A}}
\newcommand{\cB}{{\cal B}}

\newcommand{\cF}{{\cal F}}

\newcommand{\cP}{{\cal P}}

\newcommand{\cS}{{\cal S}}

\newcommand{\cU}{{\cal U}}
\newcommand{\cV}{{\cal V}}

\newcounter{todocounter}

\newtheorem{definition-a}{Definition} % [section]
\newtheorem{example-a}{Example} % [section]
\newtheorem{theorem-a}{Theorem} % [section]
\newtheorem{lemma-a}{Lemma} % [section]

%% file: introduction.tex
\section{Introduction}

The problem of logical representation for word puzzles has recently received
considerable attention
\cite{PonnuruFMT04,Shapiro11,BaralD12,schwitter2013jobs}.
In all of these studies, however, the input information is assumed to be
consistent and the proposed logical representations break on
inconsistent input.
The present paper proposes an approach that works in
the presence of inconsistency and not just for word puzzles.

At first sight, one might think that the mere use of a paraconsistent
logic such as Belanp's four valued logic \cite{belnap1977useful} or
Annotated Logic Programming \cite{subrahmanian-tcs-89,kifer1992theory}
is all what is needed to address the problem, but it is not so.
We do start with a well-known paraconsistent logic, called
\emph{Annotated Predicate Calculus} (APC) \cite{kifer1992logic},
which is related to the aforementioned Annotated Logic Programs, but
this is not enough:  a number of issues arise in the presence of
paraconsistency and different
translations might seem equivalent but behave differently when inconsistent
information is taken into account. As it turns out, several factors can affect
the choice of the ``right'' logical representation for many
NL sentences, especially for implications.
We formalize several principles to guide the translation
of NL sentences into APC, principles that can be incorporated
into existing controlled language translators, such as 
Attempto Controlled English (ACE) \cite{fuchs2008attempto} and 
PENG Light \cite{white2009update}. We illustrate these issues with the
classical Jobs Puzzle \cite{Wos_Overbeck_Lusk_Boyle_1984} and show how
inconsistent information affects the conclusions.

To address the above problems formally, we introduce
a new kind of non-monotonic semantics for APC, which is based on
\emph{consistency-preferred stable models} and is
inspired by the concept of the
most epistemically-consistent models of \cite{kifer1992logic}. We argue
that this new semantics makes APC into a good platform for
dealing with inconsistency in word puzzles and, more generally, for
translating natural language sentences into logic.

Finally, we show that the consistency-preferred 
stable models of APC can be computed
using answer-set programming (ASP) systems that support preferences over stable models, such as
Clingo \cite{gekakaosscsc11a} with the Asprin add-on
\cite{brewka2015asprin}.

This paper is organized as follows.
Section \ref{apc_section} provides background material on APC.
In Section~\ref{apc_lp} we consider the logic programming subset of APC and
define preferential stable models for it.
In Section \ref{apc_asp}, we show that the logic programming subset of APC
(under the consistency-preferred stable model semantics)
can be encoded in ASP in semantically-preserving way.
In Section \ref{puzzle}, we discuss variations of Jobs Puzzle
\cite{Wos_Overbeck_Lusk_Boyle_1984} 
when various kinds of inconsistency are injected into the
formulation of the puzzle.
Section~\ref{puzzlePrinciples} explains that logical encoding of
common knowledge in the presence of inconsistency needs to take into
account a number of considerations that are not present when inconsistency
is not an issue.  We organize those considerations into several different
principles and illustrate their impact.
Section~\ref{sec-conclusion} concludes the paper.
Finally,
Appendix A contains
the full encoding of Jobs Puzzle in APC under
the consistency-preferred semantics.
This appendix also includes variations that inject various kinds of
inconsistency into the puzzle, and the derived conclusions are discussed.
Appendices B 
and  C 
contain similar analyses of other well-known puzzles:
Zebra Puzzle\footnote{\url{https://en.wikipedia.org/wiki/Zebra_Puzzle}}
and 
Marathon Puzzle \cite{Guer00}.
Ready-to-run encodings of these programs in Clingo/Asprin
can be found at \url{https://bitbucket.org/tiantiangao/apc_lp}.
% Zebra puzzle TPTP http://www.cs.miami.edu/~tptp/cgi-bin/SeeTPTP?Category=Problems&Domain=PUZ&File=PUZ010-1.p

%%% Local Variables: 
%%% mode: latex
%%% TeX-master: "main"
%%% End: 

%% file: apc.tex
\section{Annotated Predicate Calculus: Background and Extensions} \label{apc_section}

To make this paper self-contained, this section provides the necessary
background on APC. At the end of the section, we define new 
semantic concepts for APC, which will be employed in later sections.

The \emph{alphabet} of APC consists of countably-infinite sets of: \emph{variables}
$\cV$, \emph{function symbols} $\cF$ (each
symbol having an arity; \emph{constants} are viewed as 0-ary function symbols),
\emph{predicate symbols} $\cP$, \emph{truth annotations},
quantifiers, and logical connectives.
In \cite{kifer1992logic}, truth annotations could come from an arbitrary upper
semilattice (called ``belief semilattice'' there), but here we will use only $\perp$ (unknown),
\TV{f} (false), \TV{t} (true) and $\top$
(contradiction or inconsistency), which are partially ordered as follows:
$\perp \le \TV{f} \le \top$  and $\perp \le \TV{t} \le \top$. 
$Terms$ in APC are constructed exactly as in predicate calculus: from
constants, variables and function symbols.  A \emph{ground term}
is one that has no variables.

\begin{definition}[\bfseries Atomic formulas \cite{kifer1992logic}] 
A $predicate\ term$ has the form  $p(t_{1}, t_{2}, \ldots, t_{n})$,
where $p$ is a n-ary predicate symbol and $t_1$, $t_2$, \ldots, $t_n$ are terms.
An APC \emph{atomic formula} (or an APC predicate) has the form
$p(t_{1}, t_{2}, \ldots, t_{n}) : \TV{s}$, where $p(t_{1}, t_{2}, \ldots, t_{n})$ is a predicate term and $\TV{s}$ is annotation
indicating the degree of belief (or truth) in the predicate term.
A \emph{ground atomic formula} is an atomic formula that has no variables. 
\qed
\end{definition}

We call an atomic formula of the form $p:\TV{s}$ a \TV{t}-\textit{predicate}
(resp., an \TV{f}-,
$\top$\textit{-}, or $\perp$-predicate) if \TV{s} is 
\TV{t} (resp., \TV{f}-, $\top$\textit{-}, or $\perp$).

APC includes the usual universal and existential quantifiers,
the connectives, $\wedge$ and $\vee$, and there are two negation and
two implication connectives: the \textit{ontological negation} $\neg$ and
\emph{ontological implication}  $\leftarrow$, plus
the \textit{epistemic negation} $\sim$ and
\emph{epistemic implication} $\epiarrow$.
As will be seen later, the distinction between the ontological and the
epistemic connectives is useful because they behave
differently in the presence of inconsistency.

\begin{definition}[\bfseries APC well-formed formulas \cite{kifer1992logic}] 
An APC  \emph{well-formed\ formula} is defined inductively as follows:
\begin{enumerate}[leftmargin=1cm]
\item[--] an atomic formula $p(t_{1}, t_{2}, \ldots, t_{n})$ : \TV{s}
\item[--] if $\phi$ and $\psi$ are well-formed formulas, then so are 
$\sim\phi$,
$\neg \phi$,
$\phi \wedge \psi$,
$\phi \vee \psi$,
$\phi \leftarrow \psi$, and 
$\phi \epiarrow \psi$.
\item[--] if $\phi$ is a formula and $X$ is a variable, then ($\forall X \phi$) and ($\exists X \phi$) are formulas.
\qed
\end{enumerate}
\end{definition}

%%The APC language consists of the set of all APC well-formed formulas
%%constructed from the symbols of the alphabet.
%%An APC $program$ is a finite set of well-formed formulas.

%\begin{definition}[\bfseries APC literal, epistemic literal]
An APC \textit{literal} is either a predicate $p : \TV{s}$ or an ontologically negated predicate $\neg p : \TV{s}$. An \textit{epistemic\ literal} is either a predicate $p : \TV{s}$ or an epistemically negated predicate $\sim p : \TV{s}$.
%\qed
%\end{definition}

In \cite{kifer1992logic}, the semantics was defined with respect to general
models, but here we will be dealing with logic programs and the
Herbrand semantics will be more handy.

\begin{definition}[\bfseries APC Herbrand universe, base, and interpretations]
The \textit{Herbrand\ universe} $\cU$  for APC is the set of all ground terms. 
The \textit{Herbrand\ base} $\cB$ for APC  is the set of all ground APC atomic
formulas.
An \textit{Herbrand\ interpretation} $I$ for APC  is a non-empty subset of the
Herbrand base that is \emph{closed} with respect to the following
operations: 
\begin{enumerate}[leftmargin=1cm]
\item[--] if $p : \TV{s}\in I$, then also $p : \TV{s}' \in I$ for all
  $\TV{s}' \leq \TV{s}$; and
\item[--] if $p : \mathbf{s}_{1},\,p : \mathbf{s}_{2}\in I$,
  and $\TV{s}\ =\ lub_{\leq}(\mathbf{s}_{1},\mathbf{s}_{2})$
  then $p : \TV{s}\in I$.
\end{enumerate} 
The annotations used in APC form a lattice (in our case a 4-element lattice) with the order $\leq$
and with $lub_\leq$ used as the least upper bound operator of that
lattice.

\noindent
We will also use
%$\cB_{\top}$, $\cB_{t}$, $\cB_{f}$, and $\cB_{\perp}$ to denote the subsets of
%all $\top$-,
%\TV{t}-, \TV{f}-, and $\perp$-\textit{predicates} in $\cB$.
$\cB_{\top}$ to denote the subset of all $\top$-\textit{predicates} in $\cB$. \qed
\end{definition}

As usual, a {\em variable assignment} is a mapping $\nu:\,\cV ~ \rightarrow
~ \cU$ that takes a variable and returns a ground term.
This mapping is extended to terms as follows:
$\nu(f(t_1,\,\ldots\,,\, t_n)) = f(\nu(t_1) ,\,\ldots\,,\, \allowbreak \nu(t_n))$.
We will disregard variable assignments for formulas with no free
variables (called \textit{sentences}) since they do not affect ground formulas.

\begin{definition}[\bf APC Herbrand Models] \label{def-truth-val-2}
Let $I$ be an APC Herbrand interpretation and $\nu$ be a variable assignment.  
For an atomic formula
$p(t_1, t_2,\ldots,t_n) : \TV{s}$, we write $I \models_\nu p(t_1, t_2,\ldots,t_n) : \TV{s}$ if and only if 
$p(\nu(t_1), \nu(t_2),\ldots,\nu(t_n)) : \TV{s} \in I$. For well-formed formulas $\phi$ and $\psi$, we write:
\begin{enumerate}[leftmargin=1cm]
    \item[--] $I\models_\nu \phi \wedge \psi$ if and only if $I\models_\nu \phi$ and $I\models_\nu \psi$;
    \item[--] $I\models_\nu \phi \vee \psi$ if and only if $I\models_\nu \phi$ or $I\models_\nu \psi$;
    \item[--] $I\models_\nu \neg \phi$ if and only if not $I\models_\nu \phi$;
    \item[--] $I\models_\nu (\forall X)\phi$ if and only if $I\models_{\nu'} \phi$, for every assignment $\nu'$ that differs from $\nu$ only in its $X$-value;
    \item[--] $I\models_\nu (\exists X)\phi$ if and only if $I\models_{\nu'} \phi$, for some $\nu'$ that differs from $\nu$ only in its $X$-value;
   \item[--] $I\models_\nu \psi \leftarrow \phi$ if and only if $I\models_\nu   \neg \phi \vee \psi$;
% epistemic 
   \item[--] $I\models_\nu \sim p : \TV{s}$ if and only if $I\models_\nu p : \sim \TV{s}$, 
where 
 $\sim\TV{t}\ =\ \TV{f}$, $\sim\TV{f}\ =\ \TV{t}$, $\sim \top\ =\ \top$ and $\sim\perp\ =\ \perp$;
\end{enumerate}
We also define:~
     $\sim \neg \phi \,\equiv\, \neg \sim \phi$,~
    $\sim (\phi \wedge \psi) \,\equiv\, \sim \phi \vee \sim \psi$,~
    $\sim (\phi \vee \psi) \,\equiv\,  \sim \phi \wedge \sim \psi$,~
    \linebreak[4]
   $\sim \forall X \phi \,\equiv\, \exists X \sim \phi$,~
   $\sim \exists X \phi \,\equiv\, \forall X \sim \phi$,~ and~
   $\psi \epiarrow \phi \,\equiv\, \sim \phi \vee \psi$.

A formula $\phi$ is \emph{satisfied} by $I$ if and only if $I \models_\nu \phi$
for every valuation $\nu$. In this case we write simply $I \models \phi$.
$I$ is a \emph{model} of a set of formulas $P$ if and only if every formula $\phi \in P$ is satisfied in $I$.
A set of formulas $P$ logically entails a formula $\psi$, denoted $P \models \psi$, if and only if every model of $P$ is also a model of $\psi$.
\qed
\end{definition}

APC has
two types of logical entailment: \textit{ontological} and
\textit{epistemic}. Ontological entailment is the entailment $\models$,
which we have just defined.
Before defining the \emph{epistemic} entailment, we motivate it with a number
of examples. To avoid clutter, in all examples we will only show the highest annotation for each APC predicate. For instance, if a model contains $p : \top$, then we will not show $p : \TV{t}$, $p : \TV{f}$, or $p : \perp$.

\begin{example} \label{ont_modus_ponens}
Consider the following set of APC formulas 
$P\ =\ \{
 q : \TV{t} \leftarrow p : \TV{t}, ~~
 p : \TV{t}\}$. 
It has four models: 
$m_{1}\ =\ \{p : \TV{t}, ~~ q : \TV{t}\}$, 
$m_{2}\ =\ \{p : \top, ~~ q : \TV{t}\}$, 
$m_{3}\ =\ \{p : \TV{t}, ~~ q : \top\}$ and 
$m_{4}\ =\ \{p : \top, ~~ q : \top\}$. 
Thus, $P\models q : \TV{t}$ holds (since $q : \TV{t}$ occurs in every model of $P$).
\qed
\end{example}

\begin{example} \label{ont_top}
The APC set of formulas
$P\ =\ \{
 q : \TV{t} \leftarrow p : \TV{t}, ~~
 p : \top\}$ 
has two models:
$m_{1}\ =\ \{p : \top, ~~ q : \TV{t}\}$ and 
$m_{2}\ =\ \{p : \top, ~~ q : \top\}$. 
Therefore, $P\models q : \TV{t}$ holds. \qed
\end{example}

\begin{example} \label{epi_modus_ponens}
This set of formulas
$P\ =\ \{
q : \TV{t} \epiarrow p : \TV{t}, ~~
p : \TV{t}\}$ is similar to that in Example~\ref{ont_modus_ponens}
except that it uses epistemic implication instead of the ontological one.
One of the models of that set is $m\ =\ \{p : \top, ~~ q : \perp\}$ and therefore $P\not\models q : \TV{t}$.
\qed
\end{example} 

Examples \ref{ont_modus_ponens} and \ref{ont_top} show that
\textit{ontological\ implication} has the modus ponens property, but it may
be too strong, as it allows one to draw conclusions from inconsistent
information. 
\textit{Epistemic\ implication} of Example \ref{epi_modus_ponens}, on the
other hand, is too cautious and does \emph{not} have the modus ponens property.
However, epistemic implication \emph{does} have the modus ponens
property \emph{and} it blocks drawing conclusions from inconsistency under
the \emph{epistemic entailment}, defined next.

\begin{definition}[\bfseries Most \textit{e}-consistent models \cite{kifer1992logic}]
A Herbrand interpretation $I_{1}$ is $more$ (or equally) \textit{e-consistent} than another interpretation $I_{2}$ (denoted $I_{1} \le_{\top} I_{2}$) if and only if 
$I_{1}\models p : \top$ 
implies
$I_{2}\models p : \top$ for every ground predicate term $p$.

A model $I$ of a set of formulas $P$ is a \textit{most\ e-consistent\
  model}, if there is no other model of $P$ that is strictly more
e-consistent than $I$.

A program $P$ \emph{epistemically} entails a formula $\psi$, denoted $P
\epimodel \psi$, if and only if every most e-consistent model of $P$
is also a model of $\psi$.
\qed
\end{definition}

Going back to Example \ref{epi_modus_ponens}, it has only one
most e-consistent model $m=\{p : \TV{t}, q :
\TV{t}\}$, so $P\epimodel q :
\TV{t}$ holds. The next example shows that $\epiarrow$ does \emph{not}
propagate inconsistency to conclusions. 

\begin{example} \label{epi_non_prop_inc}
Let $P\ =\ \{q : \TV{t} \epiarrow p :
\TV{t}, ~~p : \top \}$. Observe that $P$ has a most e-consistent model
$m\ =\ \{p : \top,~~ q : \perp\}$,
in which $q: \TV{t}$ does not hold. Therefore, $P\epinotmodel q : \TV{t}$ holds.\qed
\end{example} 

Next we observe that not all inconsistent information is created equal, as
people have different degrees of confidence in different pieces of
information. For instance, one normally would have higher
confidence in the fact that someone named Robin is a person than in
the fact that Robin is a male. Therefore, given a choice,
we would hold it less
likely that $person(robin)$ is inconsistent than that
$male(robin)$ is. 
Likewise, in the following example,
given a choice, we are more likely to hold to a belief that
Pete is a person than to a belief that he is rich. 

\begin{example}\label{cpr}
Consider the following formulas
\begin{enumerate}[leftmargin=1cm]
\item[] $person(pete) : \TV{t}.$ 
\item[] $businessman(pete) : \TV{t}.$
\item[]  $rich(pete) : \TV{f}.$
\item[] $rich(X) : \TV{t} \epiarrow person(X) : \TV{t} ~\wedge~ businessman(X) : \TV{t}.$
\end{enumerate}
There are three \textit{most\ e-consistent\ models}:
\begin{enumerate}[leftmargin=1cm]
\item[] $m_{1}\ =\ \{person(pete) : \top, ~~ businessman(pete) : \TV{t}, ~~ rich\allowbreak(pete) : \TV{f}\}$
\item[] $m_{2}\ =\ \{person(pete) : \TV{t}, ~~ businessman(pete) : \top, ~~ rich(pete) : \TV{f}\}$
\item[] $m_{3}\ =\ \{\allowbreak person(pete) : \TV{t}, ~~ businessman(pete) : \TV{t}, ~~ rich(pete) : \top\}$
\end{enumerate}
Based on the aforesaid confidence considerations, we are more likely to
believe that Pete is a person than that he is a
businessman or rich. Therefore, 
we are likely to think that the models $m_{2}$ and $m_{3}$ are better
descriptions of the real world than $m_{1}$.\qed
\end{example}

In this paper, we capture the above intuition by extending the notion of
most e-consistent models with additional preferences over models.

\begin{definition}[\bf Consistency-preference relation and consistency-preferred models]\label{def_mcpm}

A \textit{consistency preference} $S$ over interpretations,
where $S$ is a set of ground $\top$-predicates in APC, is defined as follows:

\begin{itemize}[leftmargin=1cm]
\item[--] An interpretation $I_{1}$ is \emph{consistency-preferred} over $I_{2}$ with
respect to $S$, denoted $I_{1}<_{S}I_{2}$, if and only if 
$S\cap I_1 ~\subset~ S\cap I_2$.

\item[--] Interpretation $I_{1}$ and $I_2$ are \textit{consistency-equal} with
respect to $S$, denoted $I_{1}=_{S}I_{2}$, if and only if 
$S\cap I_1 ~=~ S\cap I_2$.
\end{itemize}

A \textit{consistency-preference relation} $<_\cS$, where $\cS =
(S_{1},S_{2},\ldots,S_{n})$ is a sequence of sets of ground
$\top$-predicates, is defined as a lexicographic order composed out of the
sequence of consistency preferences $S_1, S_2, \ldots, S_{n}$. Namely, $I_{1}
<_\cS I_2$ if and only iff
there is $1 \leq i \leq n$ such that
$(\bigwedge_{1\leq j < i} (I_{1}=_{S_{j}} I_{2}))$ and
$I_{1}<_{S_{i}}I_{2}$.

A model $I$ of a set of formulas $P$ is called (most)
\emph{consistency-preferred} with respect to $<_\cS$ if $P$ has no other model
$J$ such that $J <_\cS I$.

We will always assume that $S_n = \cB_\top$ --- the set of all ground
$\top$-predicates and, therefore, any most consistency-preferred model is
also a most e-consistent one.

We use the notation $\epicmodel$ to denote \emph{epistemic entailment} with respect to most consistency-preferred models. A program $P$ \emph{epistemically entails} a formula $\psi$ with respect to a 
consistency-preference relation $<_\cS$, denoted 
$P \epicmodel \psi$, if and only if every most consistency-preferred model of $P$
is also a model of $\psi$.

\qed
\end{definition}

%%% Local Variables: 
%%% mode: latex
%%% TeX-master: "main"
%%% End: 

%% file: apclp.tex
\section{Logic Programming Subset of APC and Its Stable Models Semantics} \label{apc_lp}

In this section, we define the logic programming subset of APC, denoted
$APC_{LP}$, and give it
a new kind of semantics based on \emph{consistency-preferred stable models}.
\begin{definition}
An $APC_{LP}$ \emph{program} consists of rules of the form:
\begin{equation*}
l_{0} \vee \cdots \vee l_{m} \leftarrow l_{m+1} \wedge \cdots \wedge l_{n} \wedge \neg \ l_{n+1}\wedge \cdots \wedge \neg \ l_{k}.
\end{equation*}
where each $l_{i}$ is an epistemic literal. Variables are
assumed to be
implicitly universally quantified. An $APC_{LP}$ \emph{formula} is either a singleton epistemic literal, 
or a conjunction of epistemic literals, 
or a disjunction of them.
\qed
\end{definition}

The formula $l_{0} \vee \cdots \vee l_{m}$
is called the $head$ of the rule, and
$l_{m+1} \wedge \cdots \wedge l_{n} \wedge \neg \ l_{n+1}\wedge \cdots \wedge \neg \ l_{k}$
is the $body$ of that rule.

Recall from Section 2 that epistemic negation can be pushed inside and
eliminated via this law: $\sim p : \alpha\
\equiv\ p : \sim \alpha$, where $\sim\TV{t}\ =\ \TV{f}$,
$\sim\TV{f}\ =\ \TV{t}$, $\sim \top\ =\ \top$, and $\sim\perp\ =\
\perp$ so, for brevity, we assume that all $APC_{LP}$ programs
are transformed in this way and the epistemic negation is eliminated.

When the rule body is empty, the ontological implication symbol
$\leftarrow$ is usually omitted and the rule becomes a disjunction.  Such a
disjunction can also be represented as an epistemic implication
and sometimes this representation may be closer to a normal English
sentence.  For instance, the sentence, ``If a person is a businessman then
that person is rich,'' can be represented as an epistemic implication:
   $rich(X) : \TV{t} ~\epiarrow~ person(X) : \TV{t}~\wedge~ businessman(X) : \TV{t}$,
which
is easier to read than the
equivalent disjunction $rich(X) : \TV{t}~\vee~
person(X) : \TV{f}~\vee~ businessman(X) : \TV{f}$.

The notion of stable models for $APC_{LP}$ carries over from standard
answer set programming (ASP) with very few changes.

\begin{definition}[\bfseries The Gelfond-Lifschitz reduct for $APC_{LP}$] Let $P$ be an $APC_{LP}$
program and $M$ be a Herbrand interpretation. The reduct of $P$ w.r.t. $M$, denoted $\frac{P}{M}$, is a program free from ontological negation obtained by
\begin{enumerate}[leftmargin=1cm]
\item removing rules with $\neg p : \TV{s}$ in the body, where $M\models p : \TV{s}$; and
\item removing literals $\neg p : \TV{s}$ from all remaining rules.
\qed
\end{enumerate}
\end{definition}

\begin{definition}[\bfseries Stable\ models for $APC_{LP}$] 
A Herbrand interpretation $M$ is a $stable\ model$ of an $APC_{LP}$ program $P$ if $M$ is a minimal 
%most e-consistent 
model of $\frac{P}{M}$. Here, minimality is with respect to set inclusion.
\qed
\end{definition}

\begin{definition}[\bfseries Consistency-preferred stable models  for $APC_{LP}$]
Let $<_\cS$ be a consistency-preference relation of
Definition~\ref{def_mcpm}, where
$\cS = (S_{1},S_{2},\ldots,S_{n})$ is a sequence of sets of ground
$\top$-predicates.
An $APC_{LP}$ interpretation $M$ is a (most) \textit{consistency-preferred
  stable model} of an $APC_{LP}$ program $P$ if and only if:
\begin{enumerate}[leftmargin=1cm]
\item[--] $M$ is a stable model of $P$, and
\item[--] $M$ is a most consistency-preferred model with respect to $<_\cS$.
\end{enumerate}
\end{definition}

%%% Local Variables: 
%%% mode: latex
%%% TeX-master: "main"
%%% End: 

%% file: apc_asp.tex
\section{Embedding \texorpdfstring{$APC_{LP}$}{APCLP} into ASP} \label{apc_asp}

We now show that
$APC_{LP}$ can be isomorphically embedded in ASP extended with a model preference
framework, such as 
the Clingo system \cite{gekakaosscsc11a} with its Asprin extension
\cite{brewka2015asprin}. We then prove the correctness of this embedding, i.e.,
that it is one-to-one and preserves the semantics.
Next, we define the subset of ASP onto which $APC_{LP}$ maps.

\begin{definition}
$ASP_{\emph{truth}}$ is a subset of \emph{ASP} programs where the only predicate is 
\code{truth/2},  which is used to reify the APC predicate terms and
associate them with truth values. That is, these atoms have the form
$\emph{truth}(p,\TV{s})$, where the first argument  is 
the reification of an APC predicate term and the second argument
is one of these truth annotations: \code{t}, \code{f}, \code{top}, or
\code{bottom}.

An $ASP_{\emph{truth}}$ \emph{program} consists a set of rules of the form: 
\begin{equation*}
t_{0} \vee \cdots \vee t_{m} \leftarrow t_{m+1} \wedge \cdots \wedge t_{n} \wedge not \ t_{n+1} \wedge \cdots \wedge not \ t_{k}.
\end{equation*}
where the $t_i$'s are \code{truth/2}-predicates.

An $ASP_{\emph{truth}}$ \emph{formula} is either a singleton \code{truth/2}-predicate, a
conjunction of such predicates, or a disjunction of them.
\qed
\end{definition}

\begin{definition}\label{def-xi-asp}
The embedding of an $APC_{LP}$ program in $ASP_{\emph{truth}}$, denoted
$\Xi_{asp}$, is defined recursively as follows (where $tv_{asp}$ is the truth value mapping):
\begin{enumerate}[leftmargin=1cm]
\item[--] $tv_{asp}(\TV{t})\ =\ $ \code{t}
\item[--] $tv_{asp}(\TV{f})\ =\ $ \code{f}
\item[--] $tv_{asp}(\top)\ =\ $ \code{top}
\item[--] $tv_{asp}(\perp)\ =\ $ \code{bottom}
\item[--] $\Xi_{asp}(p : \TV{s})\ =\ $ \code{truth(p,}$tv_{asp}(\TV{s})$\code{)} 
\item[--] $\Xi_{asp}(\neg L)\ =\ \code{not\
    }\Xi_{asp}(L)$, where $L$ is an APC predicate
\item[--] $\Xi_{asp}(L\vee \phi)\ =\ \Xi_{asp}(L) \vee \Xi_{asp}(\phi)$,
  where $L$ is an APC predicate and $\phi$ is a disjunction of APC predicates
\item[--] $\Xi_{asp}(L \wedge \phi)\ =\ \Xi_{asp}(L) \wedge \Xi_{asp}(\phi)$, where $L$ is an APC literal and $\phi$ is a conjunction of APC literals
\item[--] $\Xi_{asp}(head\leftarrow body)\ =\ \Xi_{asp}(head)\leftarrow \Xi_{asp}(body)$, where $head$ (resp., $body$) denotes the head (resp., the body) of a rule.
\end{enumerate}

The embedding $\Xi_{asp}$ also applies to APC Herbrand\ interpretations:
each APC Herbrand interpretation (which is a set of APC atoms of the form
$p : \TV{s}$) is mapped to a set of $ASP_{\emph{truth}}$ atoms (of the form
\code{truth(p,}$tv_{asp}(\TV{s})$\code{)} ). 
\qed
\end{definition}

We require that each $ASP_{\emph{truth}}$ program includes the following
background axioms to match the semantics of APC:
{\tt
\begin{enumerate}
    \item[] truth(X,top) :- truth(X,t),truth(X,f).
    \item[] truth(X,t) :- truth(X,top).
    \item[] truth(X,f) :- truth(X,top).
    \item[] truth(X,bottom).
\end{enumerate}
}

\begin{lemma}\label{lem-one-one}
The embedding $\Xi{asp} : APC_{LP}\rightarrow ASP_{truth}$ is a one-to-one correspondence.
\qed
\end{lemma}
\begin{proof}
As mentioned, we can limit our attention to $\sim$-free programs. First, it is obvious that $\Xi_{asp}$ is injective on APC literals. Injectivity on APC conjunctions and disjunctions can be shown by a straightforward induction on the number of conjuncts and disjuncts. Surjectivity follows similarly because it is straightforward to define the inverse of $\Xi_{asp}$ by reversing the equations of Definition~\ref{def-xi-asp}.
\end{proof}

Next, we show the above APC-to-ASP embedding preserves 
models,
Gelfond-Lifshitz reduct,
stable models,
and also consistency preference relations.

\begin{lemma}\label{lem-xi-models}
  The models of any $ASP_{truth}$ program are closed with respect to
  $lub_{\leq}$ and downward-closed with respect to the $\leq$-ordering.
  Also, $I$ is a model of an $APC_{LP}$ program $P$ if and only if 
  $\Xi_{asp}(I)$ is a model of $\Xi_{asp}(P)$.
\end{lemma}
\begin{proof}
Recall that every $APC_{\emph{truth}}$ is required to have the four rules listed
right after Definition~\ref{def-xi-asp}. These rules obviously enforce the
requisite closures.
The second part of the lemma follows directly from the definitions.
\end{proof}

\begin{lemma} \label{lem-gf-reduct}
  $\Xi_{asp}$ preserves the Gelfond-Lifshitz reduct:~~
$\Xi_{asp}(\frac{P}{I})~~=~~ \frac{\Xi{asp}(P)}{\Xi_{asp}(I)}$.
\qed
\end{lemma}
\begin{proof} 
For every predicate $p : \textbf{s}\in P$, we have $I\models p :
\textbf{s}$ if and only if $\Xi_{asp}(I)\models \emph{truth}(p,\textbf{s})$, by
Lemma~\ref{lem-xi-models}. By the same lemma, if $r \in P$ then
$I\models p : \textbf{s}$ where $\neg p : \textbf{s} \in body(r)$ if and only if
$\Xi_{asp}(I)\models \emph{truth}(p,\textbf{s})$, where $not\ \emph{truth}(p,\textbf{s})
\in body(\Xi_{asp}(r))$. As a result, rule $r$ gets eliminated by 
Gelfond-Lifschitz reduction if and only if $\Xi_{asp}(r)$ is eliminated and a
negative literal in the body of $r$ gets dropped if and only if its image
in $\Xi_{asp}(r)$ gets dropped.
\end{proof}

\begin{lemma}\label{model_equivalence}
Let $I$ be a APC Herbrand interpretation.
$J$ is an APC Herbrand model of $\frac{P}{I}$ if and only if $~\Xi_{asp}(J)$ is a model of $~\Xi_{asp}(\frac{P}{I})$.
\qed
\end{lemma}
\begin{proof}
If $r\in P$ is a rule then $J\models head(r)$ if and only if
$\Xi_{asp}(J)\models head(\Xi_{asp}(r))$ and $J\not\models body(r)$ if and
only if $\Xi_{asp}(J)\not\models body(\Xi_{asp}(r))$. Thus, $J\models r$ if
and only if $\Xi_{asp}(J)\models \Xi_{asp}(r)$.
\end{proof}

\begin{lemma}\label{subset_equivalence}
Let $I_{1}$ and $I_{2}$ be APC \textit{Herbrand\ interpretations}. $I_{1}\subseteq I_{2}$ if and only if~ $\Xi_{asp}(I_{1})\subseteq \Xi_{asp}(I_{2})$.
\qed
\end{lemma}
\begin{proof}
  Follows directly from the definition of $\Xi_{apc}$ and its inverse.
\end{proof}

\begin{theorem} \label{th-stablemod}
$M$ is a stable model of an $APC_{LP}$ program $P$ if and only if ~$\Xi_{asp}(M)$ is a stable model of $\Xi_{asp}(P)$.
\qed
\end{theorem}
\begin{proof}
By Lemma \ref{model_equivalence}, $J$ is a model of $\frac{P}{M}$ if and
only if $\Xi_{asp}(J)$ is a model of
$\frac{\Xi_{asp}(P)}{\Xi_{asp}(M)}$. Thus, the set of models for
$\frac{P}{M}$ is in a one-one correspondence with the set of models for
$\frac{\Xi_{asp}(P)}{\Xi_{asp}(M)}$. By
Lemma~\ref{subset_equivalence}, this correspondence preserves
set-inclusion, so the set of \emph{minimal} models of
$\frac{P}{M}$ stands in one-one correspondence with respect to $\Xi_{asp}$
with the set of minimal models of $\frac{\Xi_{asp}(P)}{\Xi_{asp}(M)}$.
\end{proof}

A consistency preference relation $<_\cS$, where $\cS =
(S_{1},S_{2},\ldots,S_{n})$, is translated into the following Asprin
\cite{brewka2015asprin} $lexico$ preference relation $\cA$ along with
several \texttt{subset} preferences relations, each corresponding to one of
the $<_{S_i}$ that are part of $<_\cS$ (see Definition~\ref{def_mcpm}).

\code{\#preference(} $\cA$ \code{, lexico)\{} 
\code{1::name(} $s_1$ \code{);}
\code{2::name(} $s_2$ \code{);} $\ldots$\code{;}
\code{n::name(} $s_n$ \code{)\}.}

\code{\#preference(} $s_1$ \code{, subset)\{} the list of elements in $\Xi_{asp}(S_1)$ \code{\}. } 

%\code{\#preference(} $s_2$ \code{, subset)\{} the list of elements in $\Xi_{asp}(S_2)$ \code{\}. } 

$\ldots$

\code{\#preference(} $s_n$ \code{, subset)\{} the list of elements in $\Xi_{asp}(S_n)$ \code{\}. } 

\begin{lemma} \label{pref_equivalence}
Let $I_{1}$ and $I_{2}$ be APC \textit{Herbrand\ interpretations},
$\cS = (S_{1},S_{2},\ldots,S_{n})$ be a consistency preference relation and
$\cA$ be its corresponding Asprin preference relation. 
$I_{1}<_{\cS}I_{2}$ if and only if $\Xi_{asp}(I_{1})$ is preferred over $\Xi_{asp}(I_{2})$ with respect to $\cA$.
\qed
\end{lemma}
\begin{proof}
  The definition in the Asprin manual of the Asprin \emph{lexico} and 
  \emph{subset} preference relations, as applied to our preference
   statements $\cA$
  given just prior to Lemma~\ref{pref_equivalence}, 
  is just a paraphrase of the lexicographical
  consistency-preference relation $\cS$ in Definition~\ref{def_mcpm}.
  The lemma now follows from the obvious fact that $\Xi_{asp}$ maps
  $\top$-literals of ASP onto the \texttt{top}-literals of $ASP_{\emph{truth}}$,
  which have the form $\emph{truth}(p,\code{top})$.
\end{proof}

\begin{theorem} \label{th-most-preferred}
$M$ is a $most\ consistency\ preferred\ stable\ model$ of an $APC_{LP}$ program $P$ with respect to a 
consistency preference relation $<_\cS$ (where $\cS =
(S_{1},S_{2},\ldots,S_{n})$)
if and only if ~$\Xi_{asp}(M)$ is a preferred model of $\Xi_{asp}(P)$ with respect to the corresponding Asprin preference relation $\cA$.
\qed
\end{theorem}
\begin{proof}
By Lemma \ref{model_equivalence}, $J$ is a model of
$\frac{P}{M}$ if and only if $\Xi_{asp}(J)$ is a model of
$\frac{\Xi_{asp}(P)}{\Xi_{asp}(M)}$.  Since, by Lemma~\ref{pref_equivalence},
$\Xi_{asp}$ maps the preference relation $\cS$ over the APC models
into the preference relation $\cA$ over the ASP models, the result
follows.
\end{proof}

%%% Local Variables: 
%%% mode: latex
%%% TeX-master: "main"
%%% End: 

%% file: puzzle.tex
\section{Jobs Puzzle and Inconsistency} \label{puzzle}

Jobs Puzzle \cite{Wos_Overbeck_Lusk_Boyle_1984} is a classical logical
puzzle that became a benchmark of sorts for 
many automatic theorem provers \cite{Shapiro11,schwitter2013jobs};
it is also included in TPTP.\footnote{
  Thousands of Problems for Theorem Provers
(\url{http://www.cs.miami.edu/~tptp/}).
}
The usual description of Jobs Puzzle does not include
implicit knowledge, like the facts that
a person is either a male or a female (but not both),
the husband of a person must be unique, etc., so we add this knowledge
explicitly, like \cite{schwitter2013jobs}. We also changed the name Steve to Robin in order to better illustrate one form of inconsistency.

\begin{enumerate}
\item There are four people: Roberta, Thelma, Robin and Pete.
\item Among them, they hold eight different jobs.
\item Each holds exactly two jobs.
\item The jobs are: chef, guard, nurse, telephone operator, police officer (gender not implied), teacher, actor, and boxer.
\item The job of nurse is held by a male.
\item The husband of the chef is the telephone operator.
\item Roberta is not a boxer.
\item Pete has no education past the ninth grade.
\item Roberta, the chef, and the police officer went golfing together.
\end{enumerate}

In sum there are four people and eight jobs and to solve the puzzle one
must figure out who holds which jobs. The solution is that
Thelma is a chef and a boxer (and is married to Pete). Pete is a telephone
operator and an actor. Roberta is a teacher and a guard. Finally,
Robin is a police officer and a nurse.

However, if we inject inconsistency into the puzzle, current logical
approaches fail because they are based on logics
that do not tolerate inconsistency. Consider the following examples.

\begin{example}
  Let us add to the puzzle that ``Thelma is an actor.'' Given
  that the original puzzle implies that Thelma is \emph{not} an actor (she
  was a chef and a boxer), this addition causes inconsistencies.  A
  first-order encoding of the puzzle (as, say, in TPTP) or an ASP-based one
  in \cite{schwitter2013jobs} will not find any models. In contrast, an
  encoding in $APC_{LP}$ can isolate inconsistent information.
  There are two possibilities: one where Thelma is an actor and the
  other where Thelma is a female. If we add background knowledge
  that Thelma is a female's name, it is less likely that Thelma's gender
  will be
  inconsistent, so the only consistency-preferred model will have one
  inconsistent conclusion that Thelma is an actor, but all other true
  facts will remain consistent.
\end{example}

\begin{example}
  Consider adding the sentences ``Robin is a male name'' and ``Robin is a
  female name,'' which will imply that
  Robin is both a male and female.  The
  first-order and ASP-based encodings will, again, find no models, while an
  $APC_{LP}$-based encoding will localize inconsistency to just $male(robin) : \top$ and female$(robin) : \top$.
  \end{example}

\begin{example}
  Consider adding the sentence ``Robin is Thelma's husband.'' Since 
  the original job puzzle implies that Pete is Thelma's husband, this will
  cause inconsistency. If we add the background knowledge that husband
  is unique, 
  again, the encoding of this modified puzzle in
  $ACP_{LP}$ will localize inconsistency to just the aforesaid
  husband-facts.
\end{example}

%%% Local Variables: 
%%% mode: latex
%%% TeX-master: "main"
%%% End: 

%% file: puzzle_principles.tex
\section{Knowledge Representation Principles for Inconsistency} \label{puzzlePrinciples}

Mere encoding of Jobs Puzzle in $APC_{LP}$ is not enough because
it is not unique: when inconsistency is taken into account, more
information needs to be provided to obtain the encodings that match user intent.
The main problem is that, if inconsistency is allowed, the number of
possible worlds can grow to many hundreds even in relatively simple 
scenarios like Jobs Puzzle, and this practically annuls  the benefits of 
the switch to a paraconsistent logic.
We have already seen small examples of such scenarios at the end of
Section~\ref{apc_section}, which motivated our notion of consistency
preference, but there are more. We organize these scenarios around
six main principles.

\medskip

\begin{principle}{Contrapositive inference}\label{p_contrapositive}
Like in classical logic, contrapositive inference may be
useful for knowledge representation.
Consider the following sentences:
\begin{itemize}
\item[] If someone is a nurse, then that someone is educated.
\item[] Pete is not educated.
\end{itemize}
We could encode the first sentence as $\emph{educated}(X):\TV{t} \leftarrow
nurse(X):\TV{t}$ or as $\emph{educated}(X):\TV{t} \epiarrow
nurse(X):\TV{t}$.
Classically, the above sentences imply that Pete is not a nurse, but the
encoding of the first sentence using the ontological implication
$\leftarrow$ would not allow for that. If contrapositive inference is
required, 
epistemic implication should be used.

\begin{example}\label{ex_contrapositive}
Consider $P = \{$educated$(X) : \TV{t}\epiarrow nurse(\allowbreak X) :
\TV{t}, ~~ \sim$ educated$(pete) : \TV{t}\}.$ It has only one most
consistency preferred model with respect to $<_{\cS}$ (with $\cS =
(\cB_{\top})$), namely $m=\{$educated$(pete) : \TV{f}, ~~ nurse(pete) :
\TV{f}\}$. ~
Therefore, $P \epicmodel nurse(pete) : \TV{f}$ holds.

The above example uses contrapositive inference, but this is not always desirable. 
For instance, suppose $P' = \{male(X) : \TV{t}\leftarrow nurse(X) :
\TV{t}, ~~ \sim male(robin) : \TV{t}\}.$ Here we use ontological implication to block contrapositive inference. Observe that $P'$ has a most
consistency preferred model with respect to $<_{\cS}$, namely $m=\{male(robin) : \TV{f}, ~~ nurse(robin) :
\perp\}$. Therefore, $nurse(robin) : \TV{f}$ does not hold, and this is exactly what we want, even if Robin happens to be
not a male.\footnote{
  In the USA as opposed to the U.K.
  }
\qed
\end{example}
%%%%%%%%%%%%%%%%%%%%%%%%%%%%%%%%%%%%%

\end{principle}

\medskip

\begin{principle}{Propagation of inconsistency}\label{p_propagation_inconsistency}

  As discussed in Example \ref{ont_top}, APC gives us a choice of whether to
  draw conclusions from inconsistent information or not, and it is a useful
  choice. One way to block such inferences, illustrated in that example, is
  to use epistemic implication.
  Another way is to use the ontological implication with the $\TV{t} + \neg
  \top$ pattern in the rule body, e.g.,
  \[
  \emph{educated}(X) : \textbf{t} ~\leftarrow~ nurse(X) : \textbf{t} ~\wedge~
  \neg nurse(X) : \top.
  \]
  Both techniques block inferences from inconsistent information, but the
  second also blocks inference by contraposition, as discussed in Principle
  \ref{p_contrapositive}.
  The following examples illustrate the use of both of these methods.

\begin{example}\label{ex_epi_non_prop_inconst}
Let $P = \{$educated$(X) : \TV{t}\epiarrow nurse(\allowbreak X) :
\TV{t}, ~~ nurse(pete) : \top\}.$ Observe that there is one most
consistency preferred model with respect to $<_{\cS}$ (as before, $\cS =
(\cB_{\top})$)$:$~~
$m=\{$educated$(pete) : \perp, ~~ nurse(pete) : \top\}$. ~
Therefore, $P \epinotcmodel~$educated$(pete) : \TV{t}$.
\qed
\end{example}

\begin{example}\label{ex_ont_not_inconst}
Let $P ~ =\ \{$educated$(X) : \textbf{t} \leftarrow nurse(X) :
\textbf{t} \wedge \neg nurse(X) : \top, ~~~ nurse(pete) :
\top\}$. As in the previous example,
$P$ has a most consistency preferred model $m\ =\
\{nurse(pete) : \top, ~~~$educated$(pete) : \perp\}$~ and so~  $P ~\epinotcmodel~
$educated$(pete) : \TV{t}$.
\qed
\end{example}
In both of these examples, inconsistency is not propagated through the
rules, but Example~\ref{ex_epi_non_prop_inconst} allows for contrapositive
inference, while Example~\ref{ex_ont_not_inconst} does not. Indeed, suppose
that instead of $nurse(pete) : \top$ we had $\emph{educated}(pete) : \TV{f}$.
Then, in the first case, $nurse(pete):\TV{f}$ would be derived, while in
the second it would not.

\end{principle}

Blocking contrapositive inference and non-propagation of inconsistency can
be applied selectively to some literals but not the others. 

\begin{example}\label{ex_mix_contrapositive_propagate_inconst}
Consider the following sentence, ``if a person holds a job of nurse then that person is educated''. It can be encoded as
\begin{equation*}
\begin{array}{l}
(educated(X) : \TV{t} ~\epiarrow~ nurse(X) : \TV{t})~\leftarrow~
person(X) : \TV{t}.
\end{array}
\end{equation*}
The rule allows propagation of inconsistency through the $person$-predicate
but blocks such propagation for
the $nurse$-predicate. It also inhibits contrapositive inference
of $person(pete):\TV{f}$ if the head of the rule is falsified by
the additional facts $nurse(pete) : \TV{t}$ and educated$(pete) : \TV{f}$. However, due to the head of the rule,
contrapositive inference would be allowed for  $nurse(pete):\TV{f}$ if 
educated$(pete) : \TV{f}$ was given.
\end{example}

\medskip

\begin{principle}{Polarity}\label{p_mutual_exclusion}

This principle addresses situations such as  the sentence ``A person must
be either a male or a female, but not both''. 
When inconsistency is possible, we want to say three things: that
any person must be either a male and or a female, that these facts cannot
be unknown, and that if one of these is inconsistent then the other is too.

\begin{example}\label{ex_male_female}
Let $P$ be:
\begin{enumerate}[leftmargin=1cm]
\item[] $male(X) : \TV{t} ~\vee~ $female$(X) : \TV{t}~\leftarrow~ person(X) : \TV{t}$
\item[] $male(X) : \TV{f} ~\vee~ $female$(X) : \TV{f}~\leftarrow~ person(X) : \TV{t}$
\item[] $male(X) : \top ~\leftarrow~ person(X) : \TV{t} ~\wedge~ $female$(X) : \top$
\item[] female$(X) : \top ~\leftarrow~ person(X) : \TV{t} ~\wedge~ male(X) : \top$
\item[] $person(robin) : \TV{t}$
\end{enumerate}
Two most consistency preferred models exist, which minimize the
inconsistency of $person(robin)$:
\begin{quote}
$m_{1} = \{person(robin) : \TV{t}, ~~male(robin) : \TV{t}, ~~\emph{female}(robin) :
\TV{f}\}$, and \\
\hspace*{3mm}
$m_{2} = \{person(robin) : \TV{t}, ~~male(robin) : \TV{f}, ~~\emph{female}(robin) : \TV{t}\}$. 
\end{quote}
If we add $male(robin) : \top$ (or female$(robin) : \top$) to $P$, then only one most consistency preferred
model remains:
$m = \{male(robin) : \top,~$female$(robin) : \top,~person(robin) : \TV{t}\}$.
\qed
\end{example}

Conditional $polarity$ (or $polar\ dependency$) is generally represented as
follows
\begin{equation*}
\begin{split}
&p : \textbf{t}\ ;\ q : \textbf{t}\leftarrow \emph{condition}.\\
&p : \textbf{f}\ ;\ q : \textbf{f}\leftarrow \emph{condition}.\\
&q : \top \leftarrow \emph{condition} \wedge p : \top.\\
&p : \top \leftarrow \emph{condition} \wedge q : \top.\\
\end{split}
\end{equation*}
where \emph{condition} is a conjunction of atomic formulas and $p$, $q$ are polar facts
with respect to that condition.
\end{principle}

\medskip

\begin{principle}{Consistency preference relations}\label{p_cpf}
Recall from Example \ref{cpr} that inconsistent information is not created equal, as people have different degrees of confidence in different pieces of information. 
For example, we have more confidence that someone whom we barely know 
is a person compared to   the
information about this person's marital situation (e.g., whether a husband
exists). 
Therefore, person-facts are more likely to be consistent than marriage-facts and so we need to define consistency preference relations to specify
the degrees of confidence.  Consistency preference relations were
introduced in Definition \ref{def_mcpm}, and we already had numerous
examples of its use.  In Jobs Puzzle encoding in
Appendix A, we use one, fairly elaborate, consistency preference
 relation.  It first sets person and job information to be of the
highest degree of confidence. Then, it prefers consistency of gender information of everybody but Robin. Third, it prefers consistency of the job assignment information. And finally, it minimizes inconsistency in
general, for all facts.
\end{principle}

\medskip

\begin{principle}{Complete knowledge}\label{p_ck}

  This principle stipulates that certain information is defined completely,
  and cannot be unknown ($\bot$). But it can be inconsistent. 
  Moreover, similarly to closed world assumption, negative information is
  preferred.
  For instance, if we do not know that someone is someone's husband, we may assume
  that that person is not. Such conclusions can be specified via a rule
  like this:
\begin{equation*}
husband(X,Y) : \TV{f} ~\leftarrow~ person(X) : \TV{t} ~\wedge~ person(Y) : \TV{t} ~\wedge~ \neg husband(X,Y) : \TV{t}
\end{equation*}
Note that, unlike, say ASP, jumping to negative
conclusions is not ensured by the stable model semantics of APC and must be
given explicitly. But the advantage is that it can be done selectively.
More generally, this type of reasoning can be specified as
\begin{equation*}
p : \TV{f}~\leftarrow~ \emph{condition} ~\wedge~ \neg p : \TV{t}.
\end{equation*}
if $p$ is known to be a predicate that is defined completely
under the $condition$.
\end{principle}

\medskip

\begin{principle}{Exactly $N$}\label{p_constraints}
  This principle captures the encoding of cardinality constraints in the
  presence of inconsistency. For instance, in
  Jobs Puzzle, the sentences ``Every person holds exactly two jobs'' and
  ``Every job is held by exactly one person'' are encoded as cardinality
  constraints:
  \vspace{1mm}\\
  \hspace*{6mm}
$2\ \{hold(X,Y) : \textbf{t}~~~ \emph{if} ~~~job(Y) : \textbf{t}\}\ 2 ~\leftarrow ~ person(X) : \textbf{t}.$

$~~~~ 1\ \{hold(X,Y) : \textbf{t} ~~~ \emph{if} ~~~  person(X) : \textbf{t}\}\ 1 ~\leftarrow ~ job(Y) : \textbf{t}.$

$~~~~ hold(X,Y) : \textbf{f}  ~\leftarrow ~ person(X) : \textbf{t} ~\wedge~
job(Y) : \textbf{t} ~\wedge~ \neg hold(X,Y) : \textbf{t}.$
\vspace{1mm}\\
These constraints count both true and inconsistent $hold$-facts, but can be
easily modified to count only consistent true facts.
Note the role of the last rule, which closes off the information
being counted by the constraint. This is necessary because if,
say, Pete is concluded to hold \emph{exactly} two jobs (of an actor and a
phone operator) then there should be nothing unknown about him holding
any other job. Instead, $hold(pete,X):\TV{f} \wedge hold(pete,X):\top$ should
be true for any other job $X$.

The general form of the exactly $N$ constraint is:
\begin{equation*}
\begin{split}
&N\ \{L : \textbf{t} ~~~\emph{if}~~~ \emph{condition}\}\ N ~\leftarrow~  guard.\\
&L : \textbf{f} ~\leftarrow~ guard ~\wedge~ \emph{condition} ~\wedge~ \neg L : \textbf{t}.
\end{split}
\end{equation*}
\end{principle}

As in ASP, such statements can be represented as a number of ground disjunctive
rules.  The ``exactly $N$'' constraints can be generalized to ``at least $N$ and at most
$M$'' constraints, if we extend the semantics in the direction of
\cite{niemela}.

%%% Local Variables: 
%%% mode: latex
%%% TeX-master: "main"
%%% End: 

%% file: related.tex
\section{Comparison with Other Work}

Although a great deal of work is dedicated to 
paraconsistent logics and logical formalizations for word puzzles
separately,
we are unaware of any work that applies
paraconsistent logics to solving word 
puzzles that might contain inconsistencies. 
As we demonstrated, mere encoding of such puzzles in a paraconsistent
logic leads to an explosion of possible worlds, which is not
helpful.\footnote{
  Also see Appendix A and the ready-to-run examples at
  \url{https://bitbucket.org/tiantiangao/apc_lp}.
  }
Most paraconsistent logics
\cite{sep-logic-paraconsistent,paraBook,belnap1977useful,dacosta-74} 
deal with inconsistency
from the philosophical or mathematical point of view and do not discuss
knowledge representation.
Other paraconsistent logics 
\cite{subrahmanian-tcs-89,kifer1992theory}
were developed for definite logic programs and cannot be easily applied to
solving more complex knowledge representation problems that arise in word
puzzles.
An interesting question is whether our use of APC is essential, i.e., 
whether the notions of consistency-preferred models can be adapted to other
paraconsistent logics and the relationship with ASP can be established.
First, it is clear that such an adaptation is unlikely for proof-theoretic
approaches to inconsistency, such as \cite{dacosta-74}.  We do not know if
such an adaptation is possible for
model-theoretic approaches, such as \cite{belnap1977useful}.

On the word puzzles front,
\cite{Wos_Overbeck_Lusk_Boyle_1984} used the 
first-order logic theorem prover OTTER to solve Jobs Puzzle\footnote{\url{http://www.mcs.anl.gov/~wos/mathproblems/jobs.txt}} and
\cite{Shapiro11} represented Jobs Puzzle
in multiple logical languages:
TPTP,\footnote{\url{http://www.cs.miami.edu/~tptp/cgi-bin/SeeTPTP?Category=Problems\&Domain=PUZ\&File=PUZ019-1.p}}
Constraint Lingo \cite{FinkelMT04} layered on top of
the ASP system Smodels \cite{SyrjanenN01} as the backend, and 
the SNePS commonsense reasoning system \cite{Shapiro00}. 
More recently,
\cite{BaralD12,schwitter2013jobs} represented word puzzles using NL/CNL sentences,
and then
automatically translate them into ASP.
None of these underlying formalisms,
FOL, ASP, and SNePS, are equipped to reason in the 
presence of inconsistency. 
In contrast,
$APC_{LP}$, combined with the knowledge representation  principles
developed in Section~\ref{puzzlePrinciples},
localizes inconsistency and computes useful possible worlds.
In addition, $APC_{LP}$ has mechanisms to control how inconsistency is
propagated through inference, it allows one to prioritize inconsistent
information, and it provides several other ways to express user's intent
(through contraposition, completion of knowledge, etc.).

%%% Local Variables: 
%%% mode: latex
%%% TeX-master: "main"
%%% End: 

%% file: conclusion.tex
\section{Conclusion}
\label{sec-conclusion}

In this paper we discussed the problem of knowledge representation in the
presence of inconsistent information with particular focus on  representing
English sentences using logic, as in word puzzles
\cite{Wos_Overbeck_Lusk_Boyle_1984,Shapiro11,PonnuruFMT04,schwitter2013jobs,BaralD12}.
We have shown that a number of considerations play a role in deciding on a
particular encoding, which includes whether or not inconsistency should be
propagated through implications, relative degrees of confidence in
different pieces of information, and others.
We used the well-known Jobs, Zebra and  Marathon puzzles (see the appendices in the supplemental material)  
to illustrate many of the above issues and show how the conclusions change
with the introduction of different kinds of inconsistency into the
puzzle.

As a technical tool, we started with a paraconsistent logic called
\emph{Annotated Predicate Calculus} \cite{kifer1992logic} and then gave it
a special kind of non-monotonic semantics that is based on
\emph{consistency-preferred stable models}.  
We also showed that these models can be computed using ASP systems that
support preference relations over stable models, such as Clingo
\cite{gekakaosscsc11a} with the Asprin extension \cite{brewka2015asprin}.

For future work, we will consider additional puzzles which may suggest new 
knowledge representation principles. In addition, we will investigate ways to incorporate inconsistency into CNL systems. This will require introduction of background knowledge into these systems and linguistic cues into the grammar.

%%% Local Variables: 
%%% mode: latex
%%% TeX-master: "main"
%%% End: 

%% file: jobs_puzzle_apc.tex
\section{Jobs Puzzle in \texorpdfstring{$APC_{LP}$}{APCLP} with Inconsistency Injections} \label{jobsPuzzleAPC}

We now present a complete $APC_{LP}$ encoding of Jobs Puzzle and
highlight the principles, introduced in Section~\ref{puzzlePrinciples}, used in the encoding.
We also show several cases of inconsistency injection and discuss the
consequences.
The English sentences are based on the CNL representation of Jobs Puzzle
from Section 3 in \cite{schwitter2013jobs} where ``Steve'' is changed to
``Robin'' for the sake of an example (because Robin can be both a male and a female name).

\begin{enumerate}
    \item[1] Roberta is a person. Thelma is a person. Robin is a person. Pete is a person.
    \item[] $person(roberta) : \TV{t}.~~ person(thelma) :  \TV{t}.~~ person(robin) :  \TV{t}.~~ person(\emph{pete}) :  \TV{t}$.
    \item[2] Roberta is a female. Thelma is a female.
    \item[] $\emph{female}(roberta) :  \TV{t}.~~ \emph{female}(thelma) :  \TV{t}.$
    \item[3] Robin is male. Pete is male.
    \item[] $male(robin) :  \TV{t}.~~ male(\emph{pete}) :  \TV{t}.$
\end{enumerate}

Sentence 4 is encoded based on Principle \ref{p_mutual_exclusion}, which treats $male$ and $\emph{female}$ as polar facts.
\begin{enumerate}
\item[4] Exclude that a person is male and that the person is female.
\item[] $male(X) : \TV{t} ~\vee ~ \emph{female}(X) : \TV{t}~\leftarrow ~ person(X) : \TV{t}.$
\item[] $male(X) : \TV{f} ~\vee ~ \emph{female}(X) : \TV{f}~\leftarrow ~ person(X) : \TV{t}.$
\item[] $\emph{female}(X) : \top ~\leftarrow ~ person(X) : \TV{t} ~\wedge ~ male(X) : \top.$
\item[] $male(X) : \top ~\leftarrow ~ person(X) : \TV{t} ~\wedge ~ \emph{female}(X) : \top.$
\end{enumerate}

We encode sentences 5 and 6 using Principle \ref{p_constraints}, which
constrains the cardinality of $hold(X,\allowbreak Y)$. This will cause
second rule in
Sentence 5 to be repeated as part of encoding of Sentence 6, so we omit the
duplicate.
\begin{enumerate}
\item[5] If there is a job then exactly one person holds that job.
\item[] $1\ \{hold(X,Y) : \TV{t}~~~~ if ~~~~ person(X) : \TV{t}\}\ 1~\leftarrow ~ job(Y) : \TV{t}.$
\item[] $hold(X,Y) : \TV{f} ~\leftarrow ~ person(X) : \TV{t} ~\wedge ~ job(Y) : \TV{t} ~\wedge ~ \neg hold(X,Y) : \TV{t}.$
\item[6] If there is a person then the person holds exactly two jobs.
\item[] $2\ \{hold(X,Y) : \TV{t}~~~~ if ~~~~ job(Y) : \TV{t}\}\ 2~\leftarrow ~ person(X) : \TV{t}.$
\end{enumerate}
Encoding of the following facts is straightforward:
\begin{enumerate}
\item[7] Chef is a job. Guard is a job. Nurse is a job. Operator is a job. Police is a job. Teacher is a job. Actor is a job. Boxer is a job.
\item[] $job(\emph{chef}) : \TV{t}.~~ job(guard) : \TV{t}.~~ job(nurse) : \TV{t}.~~ job(\emph{operator}) : \TV{t}.$
\item[] $job(police) : \TV{t}.~~ job(teacher) : \TV{t}.~~ job(actor) : \TV{t}.~~ job(boxer) : \TV{t}.$
\end{enumerate}

Sentences 8-13 are encoded based on Principles \ref{p_contrapositive} and
\ref{p_propagation_inconsistency}, where contrapositive
inference and propagation of inconsistency are allowed for some
literals but not others.
Notice that it is undesirable to allow propagation of inconsistency
from \emph{person}-facts and \emph{job}-facts, since it is
unreasonable to conclude that somebody is, say, a male while being unsure
that this somebody is a person. Ditto about the jobs.
Contrapositive reasoning (say, from non-male to non-person) is also
inappropriate here because we have higher confidence in someone being a person.
 So, we use ontological implication $\leftarrow$ in the
next group of rules. 
\begin{enumerate}
\item[8] If a person holds a job as a nurse then that person is a male.
\item[] $(male(X) : \TV{t} ~\epiarrow ~ hold(X,nurse) : \TV{t})~\leftarrow
  ~ person(X) : \TV{t} ~\wedge ~ job(nurse) : \TV{t} ~\wedge ~$ \par $\quad
  \quad \quad \quad \quad \quad \quad \quad \quad \quad \quad \quad \quad
  \quad \qquad~$ $\neg person(X) : \top ~\wedge ~ \neg job(nurse) : \top.$
\item[9] If a person holds a job as an actor then that person is a male.
\item[] $(male(X) : \TV{t} ~\epiarrow ~ hold(X,actor) : \TV{t})~\leftarrow
  ~ person(X) : \TV{t} ~\wedge ~ job(actor) : \TV{t} ~\wedge ~$ \par $~~
  \qquad \qquad \qquad \qquad \qquad \qquad \qquad \qquad  \neg person(X) : \top ~\wedge ~ \neg job(actor) : \top.$
\item[10] If a first person holds a job as a chef and a second person holds
  a job as a telephone operator then the second person is a husband of the
  first person.
\item[] $(husband(Y,X) : \TV{t} ~\epiarrow ~ hold(X,\emph{chef}) : \TV{t} ~\wedge ~
  hold(Y,\emph{operator}) : \TV{t}) ~\leftarrow ~$\par $ \quad \qquad \quad person(X) : \TV{t} ~\wedge ~ job(\emph{chef}) : \TV{t} ~\wedge ~ person(Y) : \TV{t} ~\wedge ~ job(\emph{operator}) : \TV{t} ~\wedge ~$\par $ \qquad \qquad \neg person(X) : \top ~\wedge ~ \neg job(\emph{chef}) : \top ~\wedge ~ \neg person(Y) : ~\wedge ~ \neg job(\emph{operator}) : \top.$
\item[11] If a first person is a husband of a second person then the first person is male.
\item[] $(male(X) : \TV{t} ~\epiarrow ~ husband(X,Y) : \TV{t})~\leftarrow ~ person(X) : \TV{t} ~\wedge ~ person(Y) : \TV{t} ~\wedge ~$\par $~~~~ \quad \qquad \qquad \qquad \qquad \qquad \qquad \quad \quad~~ \neg person(X) : \top ~\wedge ~ \neg person(Y) : \top.$
\item[12] If a first person is a husband of a second person then the second person is female.
\item[] $(\emph{female}(Y) : \TV{t} ~\epiarrow ~ husband(X,Y) : \TV{t})~\leftarrow ~ person(X) : \TV{t} ~\wedge ~ person(Y) : \TV{t} ~\wedge ~$\par $\quad \qquad \qquad \qquad \qquad \qquad \qquad \qquad \quad \quad~~ \neg person(X) : \top ~\wedge ~ \neg person(Y) : \top.$
\item[13] Exclude that Roberta holds a job as boxer.
\item[] $hold(roberta,boxer) : \TV{f}~\leftarrow ~ job(boxer) :  \TV{t} ~\wedge ~ \neg job(boxer) : \top.$
\end{enumerate}
Encoding of the following fact is straightforward.
\begin{enumerate}
\item[14] Exclude that Pete is educated.
\item[] $\emph{educated}(\emph{pete}) : \TV{f}.$
\end{enumerate}

\noindent
Sentences 15-20 are also encoded based on Principles \ref{p_contrapositive}
and  \ref{p_propagation_inconsistency}.
\begin{enumerate}
\item[15] If a person holds a job as nurse then the person is educated.
\item[] $(\emph{educated}(X) : \TV{t} ~\epiarrow ~ hold(X,nurse) :
  \TV{t})~\leftarrow ~ person(X) : \TV{t} ~\wedge ~ job(nurse) : \TV{t}
  ~\wedge ~$ \par $~~~ \quad \qquad \qquad \qquad \qquad \qquad \qquad
  \qquad \quad \quad~~ \neg person(X) : \top ~\wedge ~ \neg job(nurse) : \top.$
\item[16] If a person holds a job as a police officer then that person is educated.
\item[] $(\emph{educated}(X) : \TV{t} ~\epiarrow ~ hold(X,police) :
  \TV{t})~\leftarrow ~ person(X) : \TV{t} ~\wedge ~ job(police) : \TV{t}
  ~\wedge ~$\par $\quad \quad \qquad \qquad \qquad \qquad \qquad \qquad
  \qquad \quad \quad~~ \neg person(X) : \top ~\wedge ~ \neg job(police) : \top.$
\item[17] If a person holds a job as a teacher then the person is educated.
\item[] $(\emph{educated}(X) : \TV{t} ~\epiarrow ~ hold(X,teacher) : \TV{t})~\leftarrow ~ person(X) : \TV{t} ~\wedge ~ job(teacher) : \TV{t} ~\wedge ~$\par $~~ \quad \qquad \qquad \qquad \qquad \qquad \qquad \qquad \qquad \quad ~~ \neg person(X) : \top ~\wedge ~ \neg job(teacher) : \top.$
\item[18] Exclude that Roberta holds a job as a chef.
\item[] $hold(roberta,\emph{chef}) : \TV{f} ~\leftarrow ~ job(\emph{chef}) : \TV{t} ~\wedge ~ \neg job(\emph{chef}) : \top.$
\item[19] Exclude that Roberta holds a job as a police officer.
\item[] $hold(roberta,police) : \TV{f}~\leftarrow ~ job(police) : \TV{t} ~\wedge ~ \neg job(police) : \top.$
\item[20] Exclude that a person holds a job as a chef and that the same person holds a job as a police officer.
\item[] $hold(X,\emph{chef}) : \TV{f} ~\vee ~ hold(X,police) : \TV{f}~\leftarrow ~ person(X) : \TV{t} ~\wedge ~ job(\emph{chef}) : \TV{t} ~\wedge ~$\par $~~ \quad \qquad \qquad \qquad \qquad \qquad \qquad \qquad \qquad job(police) : \TV{t} ~\wedge ~ \neg person(X) : \top ~\wedge ~$\par $~~ \quad \qquad \qquad \qquad \qquad \qquad \qquad \qquad \qquad \neg job(\emph{chef}) : \top ~\wedge ~ \neg job(police) : \top.$
\end{enumerate}

Next we define the consistency preference relation $<_{\cS}$,
where $\cS=(s_{1},s_{2},s_3,\cB_{\top})$, which implements Principle 4.
Here $s_1$ says that we hold greater confidence in the information about
someone being a person and something being a job than in any other kind of
information in the puzzle. That is, these facts are least likely to be inconsistent. 
Next, $s_2$ says that we are very likely to
believe that Pete is a male name and Thelma and Roberta are female names.
We are not sure about Robin, so s/he is left out in $s_2$.
The set $s_3$ says that next we are likely to believe the information on
who holds which jobs.
The last component in $\cS$, $\cB_\top$, is the usual default that prefers the most
e-consistent models.
\begin{enumerate}
\item[] $s_{1}=\{person(roberta) : \top,~~ person(thelma) : \top,~~ person(robin) : \top,~~ person(\emph{pete}) : \top,$ \par 
$\qquad ~~ job(\emph{chef}) : \top, job(guard) : \top, job(nurse) : \top, job(\emph{operator}) : \top,$ \par
$\qquad ~~ job(police) : \top, job(teacher) : \top, job(actor) : \top, job(\allowbreak boxer) : \top \}$
\item[]
\item[]  $s_{2}=\{male(\emph{pete}) : \top, ~~ \emph{female}(thelma) : \top, ~~ \emph{female}(roberta) : \top\}$
\item[] $s_{3}=\{hold(\emph{pete},\emph{chef}) : \top, \quad hold(\emph{pete},guard) : \top, \quad hold(\emph{pete},nurse) : \top, ~~ hold(\emph{pete},$ \par $\qquad ~~ \emph{operator}) : \top, ~~~~ hold(\emph{pete}, police) : \top, ~~~~ hold(\emph{pete},teacher) : \top, ~~~~ hold(\emph{pete},$ \par $\qquad ~~ actor) : \top, ~~ hold(\emph{pete},boxer) : \top, ~~ \ldots ~~ \}$ 
\end{enumerate}
It is interesting to note that if $s_3$ is not included then, in some
cases, there might be too many possibilities to solve the puzzle by allowing
\emph{hold}-predicates to be inconsistent. This is because many 
rules and constraints in the puzzle use \emph{hold} as a premise. So, without
minimizing the possibility of inconsistency in \emph{hold} those rules and
constraints become vacuously true, leading to an explosion of the possible
worlds.

The $APC_{LP}$ encoding generates one most consistency preferred model where the information concerning $hold/2$ is
\begin{enumerate}
\item[] $hold(\emph{pete},actor) : \TV{t}$ ~~~~~~~~~~~~~~~ $hold(\emph{pete},\emph{operator}) : \TV{t}$ 
~~~~~~~~~ $hold(robin,police) : \TV{t}$ 
\item[] $hold(robin,nurse) : \TV{t}$ ~~~~~~~~~~~~~ $hold(thelma,\emph{chef}) : \TV{t}$
~~~~~~~~~~~~ $hold(thelma,boxer) : \TV{t}$
\item[] $hold(roberta,guard) : \TV{t}$ ~~~~~~~~~ $hold(roberta,teacher) : \TV{t}$
\end{enumerate}
Not surprisingly,
this solution matches other approaches because so far we have not injected
inconsistency (and so, for example, the preferences $s_1$, $s_2$, and $s_3$
do not matter here).

Next, we illustrate five cases of injection of inconsistency into the
puzzle. Since complete models tend to be rather large, we show only
$hold/2$ and $\top$-predicates.
\begin{enumerate}
\item[21] Thelma is an actor --- Variation 1
\item[] $hold(thelma,actor) : \TV{t}$.
\end{enumerate}
This makes Thelma's job assignment as an actor inconsistent and we get this
model:
\begin{enumerate}
\item[] $m = \{hold(\emph{pete},boxer) : \TV{t} ~~~~~~~~~~~hold(\emph{pete},\emph{operator}) : \TV{t} ~~~~~~~~~~~hold(robin,police) : \TV{t}$ 
\item[] $\qquad ~~ hold(robin,nurse) : \TV{t} ~~~~~~~~~~hold(thelma,actor) : \top  ~~~~~~~~~~~hold(thelma,\emph{chef}) : \TV{t}$ 
\item[] $\qquad ~~ hold(roberta,teacher) : \TV{t} ~~~~~~~hold(roberta,guard) : \TV{t}\}$
\end{enumerate}
Indeed, the original puzzle implies that Thelma is \emph{not} an actor. Given that
$hold(thelma,actor) : \TV{t}$ is true, Sentence 9 will sanction two
possibilities: one where $male(thelma) : \TV{t}$ is
true and the other where $hold(thelma,actor) : \top$ is true. In the first
case, Sentence 4 will force $male(thelma) : \top$ and
$\emph{female}(thelma) : \top$ into the model. In the second case, we will have
$hold(thelma,actor) : \top$ in the model. Given that we have high confidence in Thelma's gender (preference $s_2$), we hold her gender
less likely to be inconsistent. Thus, the first case gets
eliminated.

\medskip

The next variation assumes that Robin is a female (instead of Thelma being
an actor in Variation 1).
\begin{enumerate}
\item[22] Robin is female --- Variation 2
\item[] $\emph{female}(robin) : \TV{t}$.
\end{enumerate}
There are two models:
\begin{enumerate}
\item[] $m_{1} = \{hold(\emph{pete},actor) : \TV{t} ~~ hold(\emph{pete},\emph{operator}) : \TV{t} ~~ hold(robin,police) : \TV{t}$
\item[] $\qquad ~~~ hold(robin,nurse) : \TV{t} ~~ hold(roberta,teacher) : \TV{t} ~~ hold(thelma,boxer) : \TV{t}$
\item[] $\qquad ~~~  hold(roberta,guard) : \TV{t} ~~ hold(thelma,\emph{chef}) : \TV{t}  ~~ husband(\emph{pete},thelma) : \TV{t}$ 
\item[] $\qquad ~~~ male(robin) : \top ~~ \emph{female}(robin) : \top\}$
\end{enumerate}
\begin{enumerate}
\item[] $m_{2} = \{hold(\emph{pete},actor) : \TV{t} ~~ hold(\emph{pete},\emph{operator}) : \TV{t} ~~ hold(robin,nurse) : \TV{t}$
\item[] $\qquad ~~~ hold(robin,\emph{chef}) : \TV{t} ~~ hold(thelma,boxer) : \TV{t} ~~ hold(thelma,police) : \TV{t}$ 
\item[] $\qquad ~~~ hold(roberta,guard) : \TV{t} ~~ hold(roberta,teacher) : \TV{t} ~~ husband(\emph{pete},robin) : \TV{t}$
\item[] $\qquad ~~~ \emph{female}(robin) : \top ~~ male(robin) : \top\}$
\end{enumerate}
Sentences 3 and 4 imply inconsistency regarding Robin's gender. The
first model is the same as in the original puzzle (as far as the job
assignments go).
In the second model, since Robin's gender is inconsistent
(both male and female),
it is compatible to make Robin a chef and Thelma a police
officer. Therefore, we derive that Pete is Robin's husband instead of
Thelma's.

\medskip

The next variation of the original puzzle explicitly assumes that Robin is Thelma's
husband.
\begin{enumerate}
\item[23] Robin is a husband of Thelma --- Variation 3
\item[] $husband(robin,thelma) : \TV{t}$.
\end{enumerate}
Here we need to add more background knowledge about marital relations.
For instance, that
every person can marry or be married to at most one person.
Together with Sentence 23, this
will cause inconsistency because the original
puzzle implies that Pete is Thelma's husband.
In \cite{schwitter2013jobs}, this implicit knowledge is not
stated, so it will fail to detect inconsistency. If such
background knowledge were added to Schwitter's formulation as constraints then there would
be no models.
The background knowledge we need is:
\begin{enumerate}
\item[24] A person who is a male is a husband of exactly one other person,
  or that person is null.
\item[25] A person who is a female has exactly one husband or that husband
  is null.
\item[26] Exclude that person X is a husband of Y and person Z is a husband
  of X simultaneously.
\item[27] If it is not derivable that person X is person Y's husband, then
  X is not Y's husband.
\end{enumerate}
Sentences 24 -- 26 are cardinality constraints and are encoded based on
Principle \ref{p_constraints}. 
Sentence 27 says that the information about
husbands is complete; it
is encoded based on Principle \ref{p_ck}.
Also, we block propagation of inconsistency from male and female based on Principle 2.
\begin{enumerate}
\item[] $1 ~ \{husband(X,Y) : \TV{t} ~~~~ if ~~~~ (person(Y) : \TV{t} ~ or
  ~ Y = null)\} ~1 ~\leftarrow $
  \par
  \qquad \qquad \qquad \qquad \qquad \qquad 
  $~ person(X) : \TV{t} ~\wedge~ male(X) : \TV{t} ~\wedge~ \neg male(X) : \top.$
\item[] $1 ~ \{husband(X,Y) : \TV{t} ~~~~ if ~~~~ (person(X) : \TV{t} ~ or
  ~ X = null)\} ~1 ~\leftarrow $
  \par
  \qquad \qquad \qquad \qquad \qquad \qquad 
  $~ person(Y) : \TV{t} ~\wedge~ \emph{female}(Y) : \TV{t} ~\wedge~ \neg \emph{female}(Y) : \top.$
\item[] ${:}{-} ~~ husband(X,Y) : \TV{t} ~\wedge~ husband(Z,X) : \TV{t} ~\wedge~ X ~ != ~ null.$
\item[] $husband(X,Y) : \TV{f} ~\leftarrow~ person(X) : \TV{t} ~\wedge~ person(Y) : \TV ~\wedge~ \neg husband(X,Y) : \TV{t}.$
\end{enumerate}

There are six models. When these models are projected on $hold/2$ (which
constitutes the solution to the puzzle) 
and the  $\top$-predicates, we get three distinct sets:
\begin{enumerate}
\item[] $m_{1} =\{ hold(\emph{pete},actor) : \TV{t} ~~~~~~~~~~ hold(\emph{pete},police) : \TV{t} ~~~~ hold(robin,\emph{operator}) : \TV{t}$
\item[] $\qquad ~~~ hold(robin,nurse) : \TV{t} ~~~~~ hold(roberta,teacher) : \TV{t} ~~ hold(roberta,guard) : \TV{t}$
\item[] $\qquad ~~~ hold(thelma,boxer) : \TV{t} ~~~~~~ hold(thelma,\emph{chef}) : \TV{t}  ~~~ \emph{educated}(\emph{pete}) : \top\}$ 
\end{enumerate}

\begin{enumerate}
\item[] $m_{2} =\{ hold(\emph{pete},actor) : \TV{t} ~~~~~ hold(\emph{pete},nurse) : \TV{t} ~~~~~ hold(robin,\emph{operator}) : \TV{t}$
\item[] $\qquad ~~~ hold(robin,police) : \TV{t} ~~~~~ hold(roberta,teacher) : \TV{t} ~~ hold(roberta,guard) : \TV{t}$
\item[] $\qquad ~~~ hold(thelma,boxer) : \TV{t} ~~~~~ hold(thelma,\emph{chef}) : \TV{t}  ~~~ \emph{educated}(\emph{pete}) : \top\}$ 
\end{enumerate}

\begin{enumerate}
\item[] $m_{3} =\{ hold(\emph{pete},police) : \TV{t} ~~~~~ hold(\emph{pete},nurse) : \TV{t} ~~~~~ hold(robin,\emph{operator}) : \TV{t}$
\item[] $\qquad ~~~ hold(robin,actor) : \TV{t} ~~~~~ hold(roberta,teacher)
  : \TV{t} ~~~~ hold(roberta,guard) : \TV{t}$
\item[] $\qquad ~~~ hold(thelma,boxer) : \TV{t} ~~~~~ hold(thelma,\emph{chef}) :
  \TV{t}  ~~~~ \emph{educated}(\emph{pete}) : \top\}$ 
\end{enumerate}

The puzzle originally implied that Pete is Thelma's husband. Since we now
explicitly stated that Robin is Thelma's husband, Sentences 25 and 27 will
force $husband(\emph{pete}, thelma) : \TV{f}$ to hold. By Sentence 10,
$hold(\emph{pete},\emph{operator}) : \TV{t}$ and $hold(Thelma,\emph{chef}) : \TV{t}$ cannot hold
simultaneously, so many solutions with inconsistencies in them will be
generated.
Due to the consistency preference relations, the $APC_{LP}$ encoding
will prefer the models where $\emph{educated}(\emph{pete}) : \top$  holds.

Sentences 24 and 25 sanction two possibilities for the
husband information in
each of the above models. For instance, the model $m_{1}$ corresponds to two
models out of the six models that we get; they differ only in their husband
information. Given that Pete is not an operator and Roberta is not a chef,
Pete is not necessarily Roberta's husband. Therefore, there are two
cases: one where $husband(\emph{pete},\allowbreak roberta) : \TV{t}$ holds and the
other where $husband(\emph{pete},null) : \TV{t}$ and $husband(null,roberta) :
\TV{t}$ hold. Similar considerations apply to $m_{2}$ and $m_{3}$.

The next variation applies the background knowledge about husbands from
Variation 3 to Variations 1 and 2.

\begin{enumerate}
\item[28] Thelma is an actor. --- Variation 4.1 (modification of Variation 1)
\item[] $hold(thelma,actor) : \TV{t}$.
\item[] Sentences 24 -- 27.
\end{enumerate}
There are now two models, and their projections on $hold$- and
$\top$-predicates are compatible with the solution to Variation 1. The only
difference between these two models is in \emph{husband}-predicates, so we
show only that part.
\begin{enumerate}
\item[] $m_{1} = \{husband(\emph{pete},thelma) : \TV{t}, ~~ husband(robin,roberta) : \TV{t}\}$ 
\item[] $m_{2} = \{husband(\emph{pete},thelma) : \TV{t}, ~~ husband(robin,null) : \TV{t}, ~~ husband(null,roberta) : \TV{t}\}$
\end{enumerate}
\medskip
\begin{enumerate}
\item[29] Robin is a female. --- Variation 4.2 (modification of Variation 2)
\item[] $\emph{female}(robin) : \TV{t}$.
\item[] Sentences 24 -- 27.
\end{enumerate}
There are three models and their projections on \emph{hold}-  and
$\top$-predicates are compatible with Variation 2. The only difference
with Variation 2 is in the $husband$-predicates, which we show:
\begin{enumerate}
\item[] $m_{1} =\{husband(\emph{pete},thelma) : \TV{t}, ~~ husband(robin,roberta) : \TV{t}\}$
\item[] $m_{2} = \{husband(\emph{pete},thelma) : \TV{t}, ~~ husband(robin,null) : \TV{t}, ~~ husband(null,roberta) : \TV{t}\}$
\item[] $m_{3} = \{husband(\emph{pete},robin) : \TV{t}, ~~ husband(null,thelma) : \TV{t}, ~~ husband(null,roberta) : \TV{t}\}$
\end{enumerate}
Here $m_{1}$ and $m_{2}$ correspond to the first model in Variation 2 and
$m_{3}$ corresponds to the second model there.

%%% Local Variables: 
%%% mode: latex
%%% TeX-master: "main"
%%% End: 

%% file: zebra_puzzle_apc.tex
\section{Zebra Puzzle in \texorpdfstring{$APC_{LP}$}{APCLP} with Inconsistency Injections} \label{zebraPuzzleAPC}

We now present a complete $APC_{LP}$ encoding of the original Zebra Puzzle as is described in Wikipedia.\footnote{\url{https://en.wikipedia.org/wiki/Zebra_Puzzle}} 
A slightly different version of the puzzle appears in TPTP.\footnote{\url{http://www.cs.miami.edu/~tptp/cgi-bin/SeeTPTP?Category=Problems&Domain=PUZ&File=PUZ010-1.p}}
The encoding highlights the principles introduced in Section~\ref{puzzlePrinciples} and we also discuss several cases of inconsistency injection.

\begin{enumerate}[label=(\alph*)]
\item There are five houses.
\item The Englishman lives in the red house.
\item The Spaniard owns the dog.
\item Coffee is drunk in the green house.
\item The Ukrainian drinks tea.
\item The green house is immediately to the right of the ivory house.
\item The Old Gold smoker owns snails.
\item Kools are smoked in the yellow house.
\item Milk is drunk in the middle house.
\item The Norwegian lives in the first house.
\item The man who smokes Chesterfields lives in the house next to the man with the fox.
\item Kools are smoked in the house next to the house where the horse is kept.
\item The Lucky Strike smoker drinks orange juice.
\item The Japanese smokes Parliaments.
\item The Norwegian lives next to the blue house.
\item[] Now, who drinks water? Who owns the zebra?
\item[] In the interest of clarity, it must be added that each of the five houses is painted a different color, and their inhabitants are of different national extractions, own different pets, drink different beverages and smoke different brands of American cigarettes [sic]. One other thing: in statement 6, right means your right.
\item[] \qquad \qquad \qquad \qquad \qquad \qquad \qquad \qquad \qquad --- Life International, December 17, 1962
\end{enumerate}

Zebra Puzzle implies some background knowledge, which we must add.  First,
there is a one-to-one correspondence between houses and colors (and
persons, drinks, cigarettes, and pets).  Second, we assume that
the first house stands on the
extreme left and the fifth house stands on the extreme right. The
first and the fifth house are not next to each other. A house is 
on the right of another house if the latter is to the left of the former.
Right-to and left-to imply next-to and next-to is a symmetric relation.
Next, we list the facts and the rules that encode the puzzle. For clarity
statement is related to an appropriate English statement from the puzzle or
from the implicit information mentioned above.

Sentences 1-5 provide the \emph{house-}, \emph{color-}, \emph{person-}, \emph{cigarette-}, and \emph{pet-}facts.
\begin{enumerate}
    \item[1] There are five houses:  \#1, \#2, \#3, \#4, and \#5.
    \item[] $house(1) : \TV{t}.~~~~ house(2) : \TV{t}. ~~~~ house(3) : \TV{t}. ~~~~ house(4) : \TV{t}. ~~~~ house(5) : \TV{t}.$

    \item[2] There are five colors: yellow, blue, red, ivory, and green.
    \item[] $color(yellow) : \TV{t}.~~~~ color(blue) : \TV{t}. ~~~~ 
color(red) : \TV{t}. ~~~~ color(ivory) : \TV{t}.~~~~ color(green) : \TV{t}.$
    
    \item[3] There are five people: Norwegian, Ukrainian, Englishman, Spaniard, and Japanese.
    \item[] $person(norwegian) : \TV{t}. ~~~~ person(ukrainian) : \TV{t}. ~~~~ person(englishman) : \TV{t}.$
    \item[] $person(spaniard) : \TV{t}. ~~~~ person(japanese) : \TV{t}.$
    
    \item[4] There are five drinks: water, tea, milk, orange juice, and coffee.
    \item[] $drink(\emph{water}) : \TV{t}. ~~~~ drink(tea) : \TV{t}. ~~~~ drink(milk) : \TV{t}. ~~~~ drink(orange\_juice) : \TV{t}.$ \par $drink(\emph{coffee}) : \TV{t}.$

    \item[5] There are five cigarettes: Kools, Chesterfield, Old Gold, Luck Strike, and Parliament.
    \item[] $cigarette(kools) : \TV{t}.~~~ cigarette(\emph{chesterfield}) : \TV{t}. ~~~ cigarette(old\_gold) : \TV{t}.$
    \item[] $cigarette(lucky\_strike) : \TV{t}. ~~~ cigarette(parliament) : \TV{t}.$

    \item[6] There are five pets: fox, horse, snails, dog, and zebra.
    \item[] $pet(\emph{fox}) : \TV{t}. ~~~~ pet(horse) : \TV{t}. ~~~~ pet(snails) : \TV{t}.$
    ~~~~ $pet(dog) : \TV{t}. ~~~~ pet(zebra) : \TV{t}.$
\end{enumerate}

Sentences 7-16 describe the implicit knowledge about the one-to-one
correspondence between houses, colors, persons, drinks, cigarettes, and
pets.  They are encoded as the cardinality constraints on the
predicates~ $house\_color(H,C)$, $house\_\emph{nationality}(H,N)$,
$house\_drink(H,\allowbreak D)$, $house\_smoke(H,C)$, and
$house\_pet(H,P)$. We encode Sentences 7 - 8 using Principles \ref{p_ck} and
\ref{p_constraints}. This will cause the second rule in Sentence 8 to be
repeated as part of the encoding of Sentence 7, so we omit the
duplicate. The rest of the sentences follow the same idea.
\begin{enumerate}
\item[7] Each house has exactly one color.
\item[] $1\{house\_color(H,C) : \TV{t} ~~if~~ color(C) : \TV{t}\}1 ~\leftarrow~ house(H) : \TV{t}.$
\item[] $house\_color(H,C) : \TV{f} ~\leftarrow~ house(H) : \TV{t} ~\wedge~ color(C) : \TV{t} ~\wedge~ \neg house\_color(H,C) : \TV{t}.$

\item[8] Each color is for exactly one house.
\item[] $1\{house\_color(H,C) : \TV{t} ~~if~~ house(H) : \TV{t}\}1 ~\leftarrow~  
color(C) : \TV{t}.$

\item[9] Each house is home for exactly one person.
\item[] $1\{house\_\emph{nationality}(H,N) : \TV{t} ~~if~~ person(N) : \TV{t}\}1 ~\leftarrow~ house(H) : \TV{t}.$
\item[] $house\_\emph{nationality}(H,N) : \TV{f} ~\leftarrow~ house(H) : \TV{t} ~\wedge~ person(N) : \TV{t} ~\wedge~$ \par $\qquad \qquad \qquad \qquad \qquad \qquad ~~~ \neg house\_\emph{nationality}(H,N) : \TV{t}.$
\item[10] Each person lives in exactly one house.
\item[] $1\{house\_\emph{nationality}(H,N) : \TV{t} ~~if~~ house(H) : \TV{t}\}1 ~\leftarrow~  person(N) : \TV{t}.$

\item[11] Each house has exactly one favorite drink.
\item[] $1\{house\_drink(H,D) : \TV{t} ~~if~~ drink(D) : \TV{t}\}1 ~\leftarrow~ house(H) : \TV{t}.$
\item[] $house\_drink(H,D) : \TV{f} ~\leftarrow~ house(H) : \TV{t} ~\wedge~ drink(D) : \TV{t} ~\wedge~ \neg house\_drink(H,D) : \TV{t}.$
\item[12] Each drink is drunk in exactly one house.
\item[] $1\{house\_drink(H,D) : \TV{t} ~~if~~ house(H) : \TV{t}\}1 ~\leftarrow~  drink(D) : \TV{t}.$

\item[13] Each house has exactly one brand of cigarettes.
\item[] $1\{house\_smoke(H,S) : \TV{t} ~~if~~ cigarette(S) : \TV{t}\}1 ~\leftarrow~ house(H) : \TV{t}.$
\item[] $house\_smoke(H,S) : \TV{f} ~\leftarrow~ house(H) : \TV{t} ~\wedge~ cigarette(S) : \TV{t} ~\wedge~ \neg house\_smoke(H,S) : \TV{t}.$
\item[14] Each brand of cigarettes is smoked in exactly one house.
\item[] $1\{house\_smoke(H,S) : \TV{t} ~~if~~ house(H) : \TV{t}\}1 ~\leftarrow~  cigarette(S) : \TV{t}.$

\item[15] Each house has exactly one pet.
\item[] $1\{house\_pet(H,P) : \TV{t} ~~if~~ pet(P) : \TV{t}\}1 ~\leftarrow~ house(H) : \TV{t}.$
\item[] $house\_pet(H,P) : \TV{f} ~\leftarrow~ house(H) : \TV{t} ~\wedge~ pet(P) : \TV{t} ~\wedge~ \neg house\_pet(H,P) : \TV{t}.$
\item[16] Each pet is kept in exactly one house.
\item[] $1\{house\_pet(H,P) : \TV{t} ~~if~~ house(H) : \TV{t}\}1 ~\leftarrow~  pet(P) : \TV{t}.$
\end{enumerate}

Sentences 17-20 correspond to Sentences (b)-(e) in the original puzzle. The
encoding is based on Principles \ref{p_contrapositive} and \ref{p_propagation_inconsistency}, where contrapositive
inference and propagation of inconsistency are allowed for some
literals but not others.
Notice that it is undesirable to allow propagation of inconsistency
from \emph{house}-facts, since it is
unreasonable to conclude that somebody lives in a house while being unsure that something is a house.
Contrapositive reasoning is also inappropriate here. For instance, if there
is uncertainty about the color of a house,
it is unreasonable to conclude that something is not a house because one
is much more likely to discern a house than its color.
So,  in the next group of rules,
we separate $house(H)$ from color and other facts using
use ontological implication $\leftarrow$.
\begin{enumerate}
\item[17] If the Englishman lives in a house then the color of the house is red.
\item[] $(house\_color(H,red) : \TV{t} ~\epiarrow~ house\_\emph{nationality}(H,englishman) : \TV{t})$ \par $\qquad \qquad \qquad \qquad \qquad \qquad \qquad \qquad \qquad \qquad  ~\leftarrow~ house(H) : \TV{t} ~\wedge~ \neg house(H) : \top$.
	
\item[18] If the Spaniard lives in a house then dog is kept in the house.
\item[] $(house\_pet(H,dog) : \TV{t} ~\epiarrow~ house\_\emph{nationality}(H,spaniard) : \TV{t})$ \par $\qquad \qquad \qquad \qquad \qquad \qquad \qquad \qquad \qquad ~\leftarrow~ house(H) : \TV{t} ~\wedge~ \neg house(H) : \top.$
	
\item[19] If the color of a house is green then coffee is drunk in the house.
\item[] $(house\_drink(H,\emph{coffee}) : \TV{t} ~\epiarrow~ house\_color(H,green) : \TV{t})$ \par $\qquad \qquad \qquad \qquad \qquad \qquad \qquad \qquad \qquad ~\leftarrow~ house(H) : \TV{t} ~\wedge~ \neg house(H) : \top.$

\item[20] If the Ukrainian lives in a house then tea is drunk in the house.
\item[] $(house\_drink(H,tea) : \TV{t} ~\epiarrow~ house\_\emph{nationality}(H,ukrainian) : \TV{t})$ \par $\qquad \qquad \qquad \qquad \qquad \qquad \qquad \qquad \qquad ~\leftarrow~ house(H) : \TV{t} ~\wedge~ \neg house(H) : \top.$
\end{enumerate}

Sentences 21-22 define the implicit knowledge of the \emph{right(X,Y)}
relation. We encode each sentence based on Principle
\ref{p_propagation_inconsistency}.
\begin{enumerate}
\item[21] A house numbered X is to the right of another house numbered Y if $~$X $-$ 1 = Y.
\item[] $right(X,Y) : \TV{t} ~\leftarrow~ house(X) : \TV{t} ~\wedge~ house(Y) : \TV{t} ~\wedge~ X - 1 ~=~ Y.$ \par $\qquad \qquad \qquad \quad~~ \neg house(X) : \top ~\wedge~ \neg house(Y) : \top.$
\item[22] A house numbered X is not to right of another house numbered Y if $~$X $-$ 1 $\neq$ Y. 
\item[] $right(X,Y) : \TV{f} ~\leftarrow~ house(X) : \TV{t} ~\wedge~ house(Y) : \TV{t} ~\wedge~ X - 1 ~\neq~ Y.$ \par $\qquad \qquad \qquad \quad~~ \neg house(X) : \top ~\wedge~ \neg house(Y) : \top.$\end{enumerate}

Sentences 23-25 correspond to Sentence (f). The encoding of Sentences 23-24
is straightforward. The encoding of Sentence 25 is based on Principles \ref{p_contrapositive} and \ref{p_propagation_inconsistency}.
\begin{enumerate}
\item[23] The color of the first house is not green.
\item[] $house\_color(1,green) : \TV{f}.$
\item[24] The color of the fifth house is not ivory. 
\item[] $house\_color(5,ivory) : \TV{f}.$
\item[25]  If a house is to the right of another house and the color of the former house is green then the color of the latter house is ivory.
\item[] $(house\_color(\emph{Left},ivory) : \TV{t} ~\epiarrow~ house\_color(Right,green) : \TV{t})$ \par $\qquad \qquad \qquad  ~\leftarrow~ house(\emph{Left}) : \TV{t} ~\wedge~ house(Right) : \TV{t} ~\wedge~ right(Right,\emph{Left}) : \TV{t} ~\wedge~$ \par $\qquad \qquad \qquad \qquad \neg house(\emph{Left}) : \top ~\wedge~ \neg house(Right) : \top ~\wedge~ \neg right(Right,\emph{Left}) :  \top.$
\end{enumerate}

Sentences 26-27 below
correspond to Sentences (g)-(h). The encoding is based on Principles
\ref{p_contrapositive} and \ref{p_propagation_inconsistency}.
\begin{enumerate}
\item[26] If Old Gold is smoked in a house then snails are kept in the house.
\item[] $(house\_pet(H,snails) : \TV{t} ~\epiarrow~ house\_smoke(H,old\_gold) : \TV{t})$ \par
$\qquad \qquad \qquad \qquad \qquad \qquad \qquad \qquad \qquad \qquad ~\leftarrow~ house(H) : \TV{t} ~\wedge~ \neg house(H) : \top.$
	
\item[27] If Kools is smoked in a house then the color of the house is yellow.
\item[] $(house\_smoke(H,kools) : \TV{t} ~\epiarrow~ house\_color(H,yellow) : \TV{t})$ \par $\qquad \qquad \qquad \qquad \qquad \qquad \qquad \qquad \qquad \qquad ~\leftarrow~ house(H) : \TV{t} ~\wedge~ \neg house(H) : \top.$
\end{enumerate}

Sentences 28-29 below correspond to Sentences (i)-(j). The encoding is straightforward.
\begin{enumerate}
\item[28] Milk is drunk in the middle house.
\item[] $house\_drink(3,milk) : \TV{t}.$
\item[29] The Norwegian lives in the first house.
\item[] $house\_\emph{nationality}(1,norwegian) : \TV{t}.$
\end{enumerate}

Sentences 30-31 state the implicit knowledge of the \emph{next(X,Y)} relation. We encode each sentence based on Principle \ref{p_propagation_inconsistency}.
\begin{enumerate}
\item[30] A house is next to another house if their house numbers differ by 1 
\item[] $next(X,Y) : \TV{t} ~\leftarrow~ house(X) : \TV{t} ~\wedge~ house(Y) : \TV{t} ~\wedge~ ~ |X - Y| ~=~ 1.$ \par $\qquad \qquad \qquad \quad~ \neg house(X) : \top ~\wedge~ \neg house(Y) : \top.$
\item[31] A house is not next to another house if their house numbers do not differ by 1
\item[] $next(X,Y) : \TV{f} ~\leftarrow~ house(X) : \TV{t} ~\wedge~ house(Y) : \TV{t} ~\wedge~ |X - Y| ~\neq~ 1.$ \par $\qquad \qquad \qquad \quad~ \neg house(X) : \top ~\wedge~ \neg house(Y) : \top.$
\end{enumerate}

Sentences 32-34 correspond to Sentence (k). Sentences 32-33 are
encoded based on Principle \ref{p_contrapositive}. Sentence 34 is encoded
based on Principles \ref{p_contrapositive} and \ref{p_propagation_inconsistency}.
\begin{enumerate}
\item[32] If Chesterfield is smoked in the first house then fox is kept in the second house.
\item[] $house\_pet(2,\emph{fox}) : \TV{t} ~\epiarrow~ house\_smoke(1,\emph{chesterfield}) : \TV{t}.$

\item[33] If Chesterfield is smoked in the fifth house then fox is kept in the fourth house.
\item[] $house\_pet(4,\emph{fox}) : \TV{t} ~\epiarrow~ house\_smoke(5,\emph{chesterfield}) : \TV{t}.$

\item[34] If a house is next to another house and Chesterfield is smoked in the former house then fox is kept in the latter house.
\item[] $(house\_pet(H1,\emph{fox}) : \TV{t} ~\vee~ house\_pet(H3,\emph{fox}) : \TV{t} $ $~\epiarrow~ house\_smoke(H2,\emph{chesterfield}) : \TV{t})$ \par $\qquad \qquad \qquad \qquad ~\leftarrow~ house(H1) : \TV{t} ~\wedge~ house(H2) : \TV{t} ~\wedge~ house(H3) : \TV{t} ~\wedge~$ \par $\qquad \qquad \qquad \qquad \qquad H1 \neq H3 ~\wedge~ next(H1,H2) : \TV{t} ~\wedge~ next(H2,H3) : \TV{t} ~\wedge~ $ \par $\qquad \qquad \qquad \qquad \qquad \neg house(H1) : \top ~\wedge~ \neg house(H2) : \top ~\wedge~ \neg house(H3) : \TV{t} : \top ~\wedge~$ \par $\qquad \qquad \qquad \qquad \qquad \neg next(H1,H2) : \TV{t} : \top ~\wedge~ \neg next(H2,H3) : \top.$\end{enumerate}

Sentences 35-37 correspond to Sentence (l). The encoding follows the same idea as Sentences 32-34.
\begin{enumerate}
\item[35] If Kools are smoked in the first house then horse is kept in the second house
\item[] $house\_pet(2,horse) : \TV{t} ~\epiarrow~ house\_smoke(1,kools) : \TV{t}.$

\item[36] If Kools are smoked in the fifth house then horse is kept in the fourth house.
\item[] $house\_pet(4,horse) : \TV{t} ~\epiarrow~ house\_smoke(5,kools) : \TV{t}.$

\item[37] If a house is next to another house and Kools are smoked in the former house then horse is kept in the latter house.
\item[] $(house\_pet(H1,horse) : \TV{t} ~\vee~ house\_pet(H3,horse) : \TV{t} ~\epiarrow~ house\_smoke(H2,kools) : \TV{t})$ \par $\qquad \qquad \qquad \qquad ~\leftarrow~ house(H1) : \TV{t} ~\wedge~ house(H2) : \TV{t} ~\wedge~ house(H3) : \TV{t} ~\wedge~$ \par $\qquad \qquad \qquad \qquad \qquad H1 \neq H3 ~\wedge~ next(H1,H2) : \TV{t} ~\wedge~ next(H2,H3) : \TV{t} ~\wedge~ $ \par $\qquad \qquad \qquad \qquad \qquad \neg house(H1) : \top ~\wedge~ \neg house(H2) : \top ~\wedge~ \neg house(H3) : \TV{t} : \top ~\wedge~$ \par $\qquad \qquad \qquad \qquad \qquad \neg next(H1,H2) : \TV{t} : \top ~\wedge~ \neg next(H2,H3) : \top.$
\end{enumerate}

Sentences 38-39 below correspond to Sentences (m)-(n). The encoding is based on
Principles \ref{p_contrapositive} and \ref{p_propagation_inconsistency}.
\begin{enumerate}
\item[38] If Lucky Strike is smoked in a house then orange juice is drunk
  in that house.
\item[] $(house\_drink(H,orange\_juice) : \TV{t} ~\epiarrow~ house\_smoke(H,lucky\_strike) : \TV{t})$
\item[] $\qquad \qquad \qquad \qquad \qquad \qquad \qquad \qquad \qquad \qquad ~\leftarrow~ house(H) : \TV{t} ~\wedge~ \neg house(H) : \top.$
\item[39] If the Japanese lives in a house then parliament is smoked in that house.
\item[] $(house\_smoke(H,parliament) : \TV{t} ~\epiarrow~ house\_\emph{nationality}(H,japanese) : \TV{t})$
\item[] $\qquad \qquad \qquad \qquad \qquad \qquad \qquad \qquad \qquad \qquad ~\leftarrow~ house(H) : \TV{t} ~\wedge~ \neg house(H) : \top.$
\end{enumerate}

Sentences 40-42 correspond to Sentence (o). The encoding follows the same idea as Sentences 32-34.
\begin{enumerate}	
\item[40] If the Norwegian lives in the first house then the color of the second house is blue.
\item[] $house\_color(2,blue) : \TV{t} ~\epiarrow~ house\_\emph{nationality}(1,norwegian) : \TV{t}.$

\item[41] If the Norwegian lives in the fifth house then the color of the fourth house is blue.
\item[] $house\_color(4,blue) : \TV{t} ~\epiarrow~ house\_\emph{nationality}(5,norwegian) : \TV{t}.$

\item[42] If a house is next to another house and the Norwegian lives in the former house then the color of the latter house is blue. 
\item[] $(house\_color(H1,blue) : \TV{t} ~\vee~ house\_color(H3,blue) : \TV{t}$ \par $\qquad \qquad \qquad \qquad \qquad \qquad \qquad \qquad \qquad ~\epiarrow~ house\_\emph{nationality}(H2,norwegian) : \TV{t})$ \par $\qquad \qquad \qquad \qquad ~\leftarrow~ house(H1) : \TV{t} ~\wedge~ house(H2) : \TV{t} ~\wedge~ house(H3) : \TV{t} ~\wedge~$ \par $\qquad \qquad \qquad \qquad \qquad H1 \neq H3 ~\wedge~ next(H1,H2) : \TV{t} ~\wedge~ next(H2,H3) : \TV{t} ~\wedge~ $ \par $\qquad \qquad \qquad \qquad \qquad \neg house(H1) : \top ~\wedge~ \neg house(H2) : \top ~\wedge~ \neg house(H3) : \TV{t} : \top ~\wedge~$ \par $\qquad \qquad \qquad \qquad \qquad \neg next(H1,H2) : \TV{t} : \top ~\wedge~ \neg next(H2,H3) : \top.$

\end{enumerate}

Next we define the consistency preference relation $<_{\cS}$,
where $\cS=(s_{1},s_{2},\cB_{\top})$, which implements Principle \ref{p_cpf}.
Here $s_1$ says that we hold greater confidence in the information about
the \emph{house-}, \emph{color-}, \emph{person-}, \emph{cigarette-}, and
\emph{pet-}, \emph{next}, \emph{right-}facts. 
We remind that ``greater confidence" here means that these facts are least likely to be inconsistent.  
The preference $s_2$ says
that next we are likely to be confident in the information given by the $house\_color(H,C)$, $house\_\emph{nationality}(\allowbreak H,N)$, $house\_drink(H,D)$, $house\_smoke(H,\allowbreak C)$, $house\_pet(H,P)$ facts.
The last component in $\cS$, $\cB_\top$, is the usual default that gives
preference to the most
e-consistent models.
\begin{enumerate}
\item[] $s_{1}=\{house(1) : \top,~~ \ldots,~~ house(5) : \top,~~ color(yellow) : \top,~~ \ldots, ~~ color(green) : \top,$ \par 
$\qquad ~~ person(norwegian) : \top,~~ \ldots, ~~person(japanese) : \top,~~ drink(\emph{water}) : \top, ~~\ldots,$ \par $\qquad ~~drink(coffe) : \top, ~~ cigarette(kools) : \top, ~~ \ldots, ~~cigarette(parliament) : \top,$ \par $\qquad ~~pet(\emph{fox}) : \top, ~~\ldots, ~~pet(zebra) : \top \}.$
\item[]  $s_{2}=\{house\_color(1, yellow) : \top, ~~ \ldots, ~~ house\_color(5,green) : \top,$ \par $\qquad ~~ house\_\emph{nationality}(1,norwegian) : \top, ~~\ldots, ~~ house\_\emph{nationality}(5, japanese) : \top,$ \par $\qquad ~~house\_drink(1,\emph{water}) : \top, ~~\ldots, ~~house\_drink(5,coffe) : \top,$ \par $\qquad ~~house\_smoke(1,kools) : \top, ~~\ldots, ~~ house\_smoke(5,parliament) : \top,$ \par $\qquad ~~ house\_pet(1,\emph{fox}) : \top, ~~\ldots, ~~ house\_pet(5,zebra) : \top\}$.
\end{enumerate}

To better illustrate the result, we define a single predicate tuple/6 to represent the combined information of
a house: its associated color, person, drink, cigarette, and pet.
\begin{enumerate}
\item[] $tuple(H,C,N,D,S,P) : \TV{t} ~\leftarrow~ house\_color(H,C) : \TV{t} ~\wedge~ house\_nationality(H,N) : \TV{t} ~\wedge~$ \par $\qquad \qquad \qquad \qquad \qquad \qquad house\_drink(H,D) : \TV{t} ~\wedge~ house\_smoke(H,S) : \TV{t} ~\wedge~$ \par $\qquad \qquad \qquad \qquad \qquad \qquad house\_pet(H,P) : \TV{t}.$
\end{enumerate}

There is a single most consistency preferred model:
\begin{enumerate}
\item[] $tuple(1,~yellow,~norwegian,~\emph{water},~kools,~\emph{fox}) : \TV{t}$
\item[] $tuple(2,~blue,~ukrainian,~tea,~\emph{chesterfield},~horse) : \TV{t}$
\item[] $tuple(3,~red,~englishman,~milk,~old\_gold,~snails) : \TV{t}$
\item[] $tuple(4,~ivory,~spaniard,~orange\_juice,~lucky\_strike,~dog) : \TV{t}$
\item[] $tuple(5,~green,~japanese,~\emph{coffee},~parliament,~zebra) : \TV{t}$\end{enumerate}

Not surprisingly,
this solution matches the usual correct solution
because so far we have not injected
any inconsistency and so, for example, the preferences $s_1$ and $s_2$
play no role.

Next, we illustrate three cases of injection of inconsistency into the
puzzle. Since complete models tend to be rather large, we show only
$tuple/6$ and $\top$-predicates.
\begin{enumerate}
\item[43] The Ukrainian lives in the middle house --- Variation 1
\item[] $house\_\emph{nationality}(3,ukrainian) : \TV{t}$.
\end{enumerate}
There are two models:
\begin{enumerate}
\item[] $m_{1} = \{tuple(1,~yellow,~norwegian,~\emph{water},~kools,~zebra) : \TV{t}$
\item[] $\qquad ~~~ tuple(2,~blue,~japanese,~tea,~parliament,~horse) : \TV{t}$
\item[] $\qquad ~~~ tuple(3,~ivory,~ukrainian,~milk,~old\_gold,~snails) : \TV{t}$
\item[] $\qquad ~~~ tuple(4,~green,~spaniard,~\emph{coffee},~\emph{chesterfield},~dog) : \TV{t}$
\item[] $\qquad ~~~ tuple(5,~red,~englishman,~orange\_juice,~lucky\_strike,~\emph{fox}) : \TV{t}$
\item[] $\qquad ~~~ house\_\emph{nationality}(3,~ukrainian) : \top\}$
\end{enumerate}

\begin{enumerate}
\item[] $m_{2} = \{tuple(1,~yellow,~norwegian,~tea,~kools,~zebra) : \TV{t}$
\item[] $\qquad ~~~ tuple(2,~blue,~japanese,~\emph{water},~parliament,~horse) : \TV{t}$
\item[] $\qquad ~~~ tuple(3,~ivory,~ukrainian,~milk,~old\_gold,~snails) : \TV{t}$
\item[] $\qquad ~~~ tuple(4,~green,~spaniard,~\emph{coffee},~\emph{chesterfield},~dog) : \TV{t}$
\item[] $\qquad ~~~ tuple(5,~red,~englishman,~orange\_juice,~lucky\_strike,~\emph{fox}) : \TV{t}$
\item[] $\qquad ~~~ house\_\emph{nationality}(3,~ukrainian) : \top\}$
\end{enumerate}

The puzzle originally implied that the Ukrainian lives in the second house. Therefore, Variation 1 generates an inconsistency about
the Ukrainian being in the middle house. By Sentence 20, we cannot derive $house\_drink(3,tea) : \TV{t}$ because propagation of inconsistency is blocked in this case. This sanctions two possibilities: one where $house\_drink(1,water) : \TV{t}$ and $house\_drink(2,tea) : \TV{t}$ hold and the other where $house\_drink(1,tea) : \TV{t}$ and $house\_drink(2,water) : \TV{t}$ hold.

\medskip
\begin{enumerate}
\item[44] The Lucky Strike is smoked in the middle house --- Variation 2
\item[] $house\_smoke(3,lucky\_strike) : \TV{t}$.
\end{enumerate}
Again, we have two models:
\begin{enumerate}
\item[] $m_{1} = \{tuple(1,~yellow,~norwegian,~\emph{water},~kools,~\emph{fox}) : \TV{t}$
\item[] $\qquad ~~~ tuple(2,~blue,~ukrainian,~tea,~\emph{chesterfield},~horse) : \TV{t}$
\item[] $\qquad ~~~  tuple(3,~ivory,~spaniard,~milk,~lucky\_strike,~dog) : \TV{t}$
\item[] $\qquad ~~~  tuple(4,~green,~japanese,~\emph{coffee},~parliament,~zebra) : \TV{t}$
\item[] $\qquad ~~~  tuple(5,~red,~englishman,~orange\_juice,~old\_gold,~snails) : \TV{t}$ 
\item[] $\qquad ~~~ house\_smoke(3,~lucky\_strike) : \top\}$
\end{enumerate}
\begin{enumerate}
\item[] $m_{2} = \{tuple(1,~yellow,~norwegian,~orange\_juice,~kools,~\emph{fox}) : \TV{t}$
\item[] $\qquad ~~~ tuple(2,~blue,~ukrainian,~tea,~\emph{chesterfield},~horse) : \TV{t}$
\item[] $\qquad ~~~  tuple(3,~ivory,~spaniard,~milk,~lucky\_strike,~dog) : \TV{t}$
\item[] $\qquad ~~~  tuple(4,~green,~japanese,~\emph{coffee},~parliament,~zebra) : \TV{t}$
\item[] $\qquad ~~~  tuple(5,~red,~englishman,~\emph{water},~old\_gold,~snails) : \TV{t}$ 
\item[] $\qquad ~~~ house\_smoke(3,~lucky\_strike) : \top\}$
\end{enumerate}

The puzzle originally implied that the Lucky Strike is smoked in the fourth
house. As a result, Variation 2 generates an inconsistency regarding the
Lucky Strike being smoked in the middle house. By Sentence 38, we cannot
derive $house\_drink(3,orange\_juice) : \TV{t}$ because propagation of
inconsistency is blocked by the epistemic implication. This sanctions two
possibilities: one where $house\_drink(1,water) : \TV{t}$ and
$house\_drink(5,orange\_juice) : \TV{t}$ hold and the other where
$house\_drink(1,orange\_juice) : \TV{t}$ and $house\_drink(5,water) :
\TV{t}$ hold.

\medskip

\begin{enumerate}
\item[45] Milk is not drunk in the middle house --- Variation 3
\item[] $house\_drink(3,milk) : \TV{f}$.
\end{enumerate}
Variation 3 generates the same model as Zebra Puzzle's original solution except for an additional inconsistent fact $house\_drink(3,milk) : \top$. This is because the puzzle originally implied that milk \emph{is} drunk in the middle house.

%%% Local Variables: 
%%% mode: latex
%%% TeX-master: "main"
%%% End: 

%% file: marathon_puzzle_apc.tex
\section{Marathon Puzzle in \texorpdfstring{$APC_{LP}$}{APCLP} with Inconsistency Injections} \label{marathonPuzzleAPC}

We present here a complete $APC_{LP}$ encoding of Marathon Puzzle \cite{Guer00,Schwitter12}.
As with the previous puzzles, the encoding highlights the principles of Section~\ref{puzzlePrinciples} and we also discuss several cases of inconsistency injection.

Marathon puzzle is as follows:
\begin{enumerate}[label=(\alph*)]
\item[] Dominique, Ignace, Naren, Olivier, Philippe, and Pascal have arrived as the
first six at the Paris marathon.
\item[] Olivier has not arrived last.
\item[] Dominique, Pascal and Ignace have arrived before Naren and Olivier.
\item[] Dominique who was third last year has improved this year.
\item[] Philippe is among the first four.
\item[] Ignace has arrived neither in second nor third position.
\item[] Pascal has beaten Naren by three positions.
\item[] Neither Ignace nor Dominique are in the fourth position.
\end{enumerate}

The original description implies some implicit background knowledge. First, no runners arrive at the same time (i.e., each runner has a unique arrival position and vice versa). Second, a runner arrives before another runner if the first runner's position number is smaller than the second runner's position. 

There is a unique solution for Marathon Puzzle where Ignace arrives first, followed by Dominique, Pascal, Philippe, Olivier, and Naren, in that order. Next, we show the encoding of Marathon puzzle.

\begin{enumerate}
\item[1] Dominique, Ignace, Naren, Olivier, Philippe, and Pascal have arrived as the first six at the Paris marathon.
\item[] $runner(dominique) : \TV{t}.~~~runner(naren) : \TV{t}.~~~runner(ignace) : \TV{t}.
~~~runner(olivier) : \TV{t}.$
\item[] $runner(\emph{philippe}) : \TV{t}.~~~runner(pascal) : \TV{t}.$
\item[] $\emph{position}(1) : \TV{t}.~~~\emph{position}(2) : \TV{t}.~~~\emph{position}(3) : \TV{t}.
~~~\emph{position}(4) : \TV{t}.~~~\emph{position}(5) : \TV{t}.$
\item[] $\emph{position}(6) : \TV{t}.$
\end{enumerate}

Sentences 2 and 3 are encoded based on Principles \ref{p_ck} and
\ref{p_constraints}. This will cause the second rule in Sentence 3 to be
repeated as part of the encoding of Sentence 2 so, as before, we omit the
duplicate.

\begin{enumerate}
\item[2] Every runner has exactly one position.
\item[] $1 \{ \emph{has\_position}(R,P) : \TV{t}~~~\emph{if}~~~\emph{position}(P) : \TV{t} \} 1 \leftarrow runner(R) : \TV{t}.$
\item[] $\emph{has\_position}(R,P) : \TV{f} \leftarrow runner(R) : \TV{t}~\wedge~\emph{position}(P) : \TV{t}~\wedge~\neg \emph{has\_position}(R,P) : \TV{t}.$
\item[3] Every position belongs to exactly one runner.
\item[] $1 \{ \emph{has\_position}(R,P) : \TV{t}~~~\emph{if}~~~runner(R) : \TV{t} \} 1 \leftarrow \emph{position}(P) : \TV{t}.$
\end{enumerate}	

Encoding of Sentence 4 is straightforward.
\begin{enumerate}
\item[4] Olivier has not arrived last.
\item[] $\emph{has\_position}(olivier,6) : \TV{f}.$
\end{enumerate}

Sentence 5 is encoded by three rules. 
The first and second rules are the result of applying 
 Principle \ref{p_propagation_inconsistency} defining 
the $\emph{before}$-relation based on the position of runners. 
The third rule represents complete knowledge of the $\emph{before}$-relation based on  Principle \ref{p_ck}.

\begin{enumerate}
\item[5] If a runner R1 has a position P1 and another runner R2 has a position P2 and P1 $<$ P2 then R1 is \emph{before} R2.
\item[] $\emph{before}(R1,R2) : \TV{t} \leftarrow runner(R1) : \TV{t}~\wedge~\emph{position}(P1) : \TV{t}~\wedge~\emph{has\_position}(R1,P1) : \TV{t}~\wedge$
\item[]  $\qquad \qquad \qquad \qquad~~~ runner(R2) : \TV{t}~\wedge~\emph{position}(P2) : \TV{t}~\wedge~\emph{has\_position}(R2,P2) : \TV{t}~\wedge~$
\item[]  $\qquad \qquad \qquad \qquad~~~ P1 < P2 ~\wedge~\neg runner(R1) : \top~\wedge~\neg runner(R2) : \top~\wedge~$
\item[] $\qquad \qquad \qquad \qquad~~~ \neg \emph{position}(P1) : \top~\wedge~\neg \emph{position}(P2) : \top~\wedge~$
\item[] $\qquad \qquad \qquad \qquad~~~ \neg \emph{has\_position}(R1,P1) : \top~\wedge~\neg \emph{has\_position}(R2,P2) : \top.$
\item[] $\emph{before}(R2,R1) : \TV{f} \leftarrow runner(R1) : \TV{t}~\wedge~\emph{position}(P1) : \TV{t}~\wedge~\emph{has\_position}(R1,P1) : \TV{t}~\wedge$
\item[]  $\qquad \qquad \qquad \qquad~~~ runner(R2) : \TV{t}~\wedge~\emph{position}(P2) : \TV{t}~\wedge~\emph{has\_position}(R2,P2) : \TV{t}~\wedge~$
\item[]  $\qquad \qquad \qquad \qquad~~~ P1 < P2 ~\wedge~\neg runner(R1) : \top~\wedge~\neg runner(R2) : \top~\wedge~$
\item[] $\qquad \qquad \qquad \qquad~~~ \neg \emph{position}(P1) : \top~\wedge~\neg \emph{position}(P2) : \top~\wedge~$
\item[] $\qquad \qquad \qquad \qquad~~~ \neg \emph{has\_position}(R1,P1) : \top~\wedge~\neg \emph{has\_position}(R2,P2) : \top.$
\item[]  $\emph{before}(R1,R2) : \TV{f} \leftarrow runner(R1) : \TV{t}~\wedge~runner(R2) : \TV{t}~\wedge~\neg \emph{before}(R1,R2) : \TV{t}.$
\end{enumerate}

Encoding of Sentences 6 - 9 is straightforward.
\begin{enumerate}
\item[6] Dominique, Pascal and Ignace have arrived before Naren and Olivier.
\item[] $\emph{before}(dominique,naren) : \TV{t}.~~~\emph{before}(pascal,naren) : \TV{t}.~~~\emph{before}(ignace,naren) : \TV{t}.$
\item[] $\emph{before}(dominique,olivier) : \TV{t}.~~~\emph{before}(pascal,olivier) : \TV{t}.~~~\emph{before}(ignace,olivier) : \TV{t}.$
\item[7] Dominique who was third last year has improved this year.
\item[] $\emph{has\_position}(dominique,1) : \TV{t}~\vee~\emph{has\_position}(dominique,2) : \TV{t}.$
\item[8] Philippe is among the first four.
\item[] $\emph{has\_position}(\emph{philippe},1) : \TV{t}~\vee~\emph{has\_position}(\emph{philippe},2) : \TV{t}~\vee~$
\item[] $\emph{has\_position}(\emph{philippe},3) : \TV{t}~\vee~\emph{has\_position}(\emph{philippe},4) : \TV{t}.$
\item[9] Ignace has arrived neither in second nor third position.
\item[] $\emph{has\_position}(ignace,2) : \TV{f}.$
\item[] $\emph{has\_position}(ignace,3) : \TV{f}.$
\end{enumerate}

Encoding of Sentence 10 is based on Principle \ref{p_propagation_inconsistency}.
The first rule encodes that if the position of Pascal is known, then the position of Naren is 
the position of Pascal plus 3.
The second rule encodes that if the position of Naren is known, then the position of Pascal is 
the position of Naren minus 3.

\begin{enumerate}
\item[10] Pascal has beaten Naren by three positions.
\item[] $\emph{has\_position}(naren,P2) : \TV{t} \leftarrow \emph{has\_position}(pascal,P1) : \TV{t} ~\wedge~\emph{position}(P1) : \TV{t} ~\wedge~$
\item[] $\qquad \qquad \qquad \qquad~~~~ \emph{position}(P2) : \TV{t}~\wedge~ P2 = P1 + 3  ~\wedge~ \neg \emph{position}(P1) : \top ~\wedge~$
\item[] $\qquad \qquad \qquad \qquad~~~ \neg \emph{position}(P2) : \top ~\wedge~\neg \emph{has\_position}(pascal,P1) : \top.$
\item[] $\emph{has\_position}(pascal,P1) : \TV{t} \leftarrow \emph{has\_position}(naren,P2) : \TV{t} ~\wedge~\emph{position}(P1) : \TV{t} ~\wedge~$
\item[] $\qquad \qquad \qquad \qquad~~~ \emph{position}(P2) : \TV{t}~\wedge~ P2 = P1 + 3 ~\wedge~ \neg \emph{position}(P1) : \top ~\wedge~$
\item[] $\qquad \qquad \qquad \qquad~~~~ \neg \emph{position}(P2) : \top ~\wedge~\neg \emph{has\_position}(naren,P2) : \top.$
\end{enumerate}

Encoding of Sentence 11 is straightforward.
\begin{enumerate}
\item[11] Neither Ignace nor Dominique are in the fourth position.
\item[] $\emph{has\_position}(ignace,4) : \TV{f}.~~~\emph{has\_position}(dominique,4) : \TV{f}.$
\end{enumerate}

%%%%%%%%%%%%%%%%%%%%%%%%%%%%%%%%%%%%%%%%%
Next we define the consistency preference relation $<_{\cS}$,
where $\cS=(s_{1},s_{2},\cB_{\top})$, which implements Principle \ref{p_cpf}.
Here $s_1$ says that we hold greater confidence in the information about
the \emph{runner-} and \emph{position-}facts and therefore these types of facts are less likely to be inconsistent. The preference $s_2$ says
that next we are likely to be confident in the information given by the $\emph{has\_position}$- and $\emph{before}$-facts. The last component in $\cS$, $\cB_\top$, is the usual default that gives
preference to the most e-consistent models.
\begin{enumerate}
\item[] $s_{1}=\{runner(dominique) : \top,~~ runner(naren) : \top,~~ runner(ignace) : \top,$ \par $\qquad~~ runner(olivier) : \top,~~ runner(\emph{philippe}) : \top,~~ runner(pascal) : \top,$ \par $\qquad~~ \emph{position}(1) : \top,~~\ldots,~~\emph{position}(6) : \top\}.$
\item[]  $s_{2}=\{\emph{has\_position}(dominique,1) : \top,~~\ldots, ~~\emph{has\_position}(dominique,6) : \top,~~\ldots$ \par $\qquad ~~ \emph{has\_position}(pascal,1) : \top,~~\ldots, ~~\emph{has\_position}(pascal,6) : \top,$ \par $\qquad ~~\emph{before}(dominique,naren) : \top,~~\ldots,~~ \emph{before}(dominique,pascal) : \top,~~ \ldots$ \par $\qquad ~~\emph{before}(pascal,dominique) : \top,~~\ldots,~~\emph{before}(pascal,\emph{philippe}) : \top\}$.
\end{enumerate}

There is a single most consistency-preferred model as expected. The puzzle, as stated, has no inconsistent information so we show only the
$\emph{has\_position}$-facts.

$m = \{\emph{has\_position}(ignace,1) : \TV{t},~~\emph{has\_position}(dominique,2) : \TV{t},$

$\qquad ~~\emph{has\_position}(pascal,3) : \TV{t},~~\emph{has\_position}(\emph{philippe},4) : \TV{t},$

$\qquad ~~\emph{has\_position}(olivier,5) : \TV{t},~~\emph{has\_position}(naren,6) : \TV{t}\}.$

In addition, there are many \emph{before}-facts which encode the sequential order of arrival of the runners. 

Next, we illustrate several cases of injection of inconsistency into the
puzzle. Since complete models tend to be rather large, we show only the
$\emph{has\_position}$- and $\top$-predicates.

\begin{enumerate}
\item[12] Pascal arrives in the sixth position --- Variation 1
\item[] $\emph{has\_position}(pascal,6) : \TV{t}.$
\end{enumerate}

There is one model and it contains inconsistencies:

$m = \{\emph{has\_position}(ignace,1) : \TV{t},~~\emph{has\_position}(dominique,2) : \TV{t},$

$\qquad ~~\emph{has\_position}(naren,3) : \TV{t},~~\emph{has\_position}(\emph{philippe},4) : \TV{t},$

$\qquad ~~\emph{has\_position}(olivier,5) : \TV{t},~~\emph{has\_position}(pascal,6) : \TV{t},$

$\qquad ~~\emph{before}(pascal,naren) : \top,~~ \emph{before}(pascal,olivier) : \top\}.$

The \emph{before}-facts still represent the sequential order of arrival, so $\emph{before}(pascal,naren) : \TV{t}$ and $\emph{before}(pascal,olivier) : \TV{t}$ are true in the model. However, Sentence 12 contradicts these facts, so $\emph{before}(pascal,naren) : \TV{f}$ and $\emph{before}(pascal,olivier) : \TV{f}$ are also true. 
Therefore, these facts become inconsistent. Sentence 10 does not imply any inconsistencies beyond the ones already mentioned.

\begin{enumerate}
\item[13] Ignace arrives in the second position. --- Variation 2
\item[] $\emph{has\_position}(ignace,2) : \TV{t}.$
\end{enumerate}

There is one model:

$m = \{\emph{has\_position}(dominique,1) : \TV{t},~~\emph{has\_position}(ignace,2) : \top,$

$\qquad ~~\emph{has\_position}(pascal,3) : \TV{t},~~\emph{has\_position}(\emph{philippe},4) : \TV{t},$

$\qquad ~~\emph{has\_position}(olivier,5) : \TV{t},~~\emph{has\_position}(naren,6) : \TV{t}\}.$

Given that Sentence 13 contradicts Sentence 9, the fact $\emph{has\_position}(ignace,2)$ 
becomes inconsistent and the order of arrival of Dominique and Ignace are swapped.

\begin{enumerate}
\item[14] Philippe arrives before Dominique. --- Variation 3
\item[] $\emph{before}(\emph{philippe},dominique) : \TV{t}.$
\end{enumerate}

Now, we get 4 models:

$m_{1} = \{\emph{has\_position}(ignace,1) : \TV{t},~~\emph{has\_position}(dominique,2) : \TV{t},$

$\qquad ~~~\emph{has\_position}(pascal,3) : \TV{t},~~\emph{has\_position}(\emph{philippe},4) : \TV{t},$

$\qquad ~~~\emph{has\_position}(olivier,5) : \TV{t},~~\emph{has\_position}(naren,6) : \TV{t},$

$\qquad ~~~\emph{before}(\emph{philippe},dominique) : \top\}.$

$m_{2} = \{\emph{has\_position}(\emph{philippe},1) : \TV{t},~~\emph{has\_position}(dominique,2) : \TV{t},$

$\qquad ~~~\emph{has\_position}(pascal,3) : \TV{t},~~\emph{has\_position}(ignace,4) : \top,$

$\qquad ~~~\emph{has\_position}(olivier,5) : \TV{t},~~\emph{has\_position}(naren,6) : \TV{t}\}.$

$m_{3} = \{\emph{has\_position}(\emph{philippe},1) : \TV{t},~~\emph{has\_position}(dominique,2) : \TV{t},$

$\qquad ~~~\emph{has\_position}(pascal,3) : \TV{t},~~\emph{has\_position}(olivier,4) : \TV{t},$

$\qquad ~~~\emph{has\_position}(ignace,5) : \TV{t},~~\emph{has\_position}(naren,6) : \TV{t},$

$\qquad ~~~\emph{before}(ignace,olivier) : \top\}.$

$m_{4} = \{\emph{has\_position}(\emph{philippe},1) : \TV{t},~~\emph{has\_position}(dominique,2) : \TV{t},$

$\qquad ~~~\emph{has\_position}(naren,3) : \TV{t},~~\emph{has\_position}(pascal,4) : \TV{t},$

$\qquad ~~~\emph{has\_position}(ignace,5) : \TV{t},~~\emph{has\_position}(olivier,6) : \top,$

$\qquad ~~~\emph{before}(pascal,naren) : \top,~~\emph{before}(ignace,naren) : \top\}.$

The first model generates the same solution as the original puzzle except for the inconsistency where Philippe arrives before Dominique.
The second model places Ignace in the fourth position, which contradicts Sentence 11, so placing Ignace in the fourth position becomes inconsistent.
The third model places Olivier before Ignace, which is in contradiction with Sentence 6. All the other constraints are satisfied, so no more inconsistencies are derived.
The fourth model places Naren before Pascal and Ignace, which contradicts Sentence 6. Besides, it places Olivier the last, which contradicts Sentence 4.

%%% Local Variables: 
%%% mode: latex
%%% TeX-master: "main"
%%% End: 